\documentclass{article}

\usepackage{microtype}
\usepackage{graphicx}
\usepackage{subfigure}
\usepackage{booktabs} 

\usepackage{hyperref}

\usepackage[accepted]{icml2020}

\usepackage{nicefrac}
\usepackage{xspace}
\usepackage{multirow}
\usepackage{mathtools}
\usepackage{wrapfig}
\usepackage{mkolar_definitions}

\newcount\Comments  
\definecolor{darkgreen}{rgb}{0,0.5,0}
\definecolor{darkred}{rgb}{0.7,0,0}
\definecolor{teal}{rgb}{0.3,0.8,0.8}
\definecolor{orange}{rgb}{1.0,0.5,0.0}
\definecolor{purple}{rgb}{0.8,0.0,0.8}
\newcommand{\kibitz}[2]{\ifnum\Comments=1{\textcolor{#1}{\textsf{\footnotesize #2}}}\fi}
\newcommand{\akshay}[1]{\kibitz{darkgreen}{[AK: #1]}}

\newcommand{\defeq}{\coloneqq}
\newcommand{\eqdef}{\eqqcolon}
\renewcommand{\cite}{\citep}
\def\vdm{\ensuremath{\hat{V}_{\textup{\mdseries DM}}}}
\def\vips{\ensuremath{\hat{V}_{\textup{\mdseries IPS}}}}
\def\vdr{\ensuremath{\hat{V}_{\textup{\mdseries DR}}}}
\def\vdrpi{\ensuremath{\hat{V}_{\textup{\mdseries DR-PI}}}}
\def\vdrsopi{\ensuremath{\hat{V}_{\textup{\mdseries DRos-PI}}}}
\def\vdrso{\ensuremath{\hat{V}_{\textup{\mdseries DRos}}}}
\def\vdrs{\ensuremath{\hat{V}_{\textup{\mdseries DRs}}}}
\def\vdrsp{\ensuremath{\hat{V}_{\textup{\mdseries DRps}}}}

\def\vtheta{\ensuremath{\hat{V}_{\theta}}}
\def\bias{\textup{\mdseries Bias}}
\def\hbias{\textup{\mdseries BiasUB}}
\def\hvar{\widehat{\Var}}
\def\tbias{\widetilde{\textup{\mdseries Bias}}}
\def\tvar{\widetilde{\Var}}
\def\mse{\textup{\mdseries MSE}\xspace}

\def\ndcg{\mathrm{NDCG}}
\def\dcg{\mathrm{DCG}}
\def\dcgs{\mathrm{DCG}^\star}
\def\rel{\mathrm{rel}}
\def\lasso{\textit{lasso}}
\def\tree{\textit{tree}}

\def\lone{\textup{\mdseries DRps}\xspace}
\def\ltwo{\textup{\mdseries DRos}\xspace}
\def\drs{\textup{\mdseries DRs}\xspace}
\def\drsdirect{\textup{\mdseries DRs-direct}\xspace}
\def\drsup{\textup{\mdseries DRs-upper}\xspace}
\def\snips{\textup{\mdseries snIPS}\xspace}
\def\sndr{\textup{\mdseries snDR}\xspace}
\def\switch{\textsc{switch}\xspace}

\newcommand{\heta}{\hat{\eta}}
\newcommand{\hetab}{\hat{\etab}}
\newcommand{\hw}{\hat{w}}
\newcommand{\hwb}{\hat{\wb}}

\newcommand{\E}{\mathbb{E}}
\newcommand{\Emu}[1]{\E_\mu\!\Bracks{#1}}
\DeclareMathOperator*{\Minimize}{Minimize}
\newcommand{\opt}{\textup{\mdseries o}\xspace}
\newcommand{\pes}{\textup{\mdseries p}\xspace}

\DeclareMathOperator{\linspan}{span}
\DeclareMathOperator{\supp}{supp}

\newcommand{\card}[1]{\lvert#1\rvert}

\newcommand{\set}[1]{\{#1\}}

\newcommand{\braces}[1]{\{#1\}}

\newcommand{\bigBraces}[1]{\bigl\{#1\bigr\}}
\newcommand{\BigBraces}[1]{\Bigl\{#1\Bigr\}}

\newcommand{\bigBracks}[1]{\bigl[#1\bigr]}
\newcommand{\BigBracks}[1]{\Bigl[#1\Bigr]}

\newcommand{\Bracks}[1]{\left[#1\right]}

\newcommand{\Parens}[1]{\left(#1\right)}
\newcommand{\bigParens}[1]{\bigl(#1\bigr)}
\newcommand{\BigParens}[1]{\Bigl(#1\Bigr)}

\newcommand{\given}{\mathbin{\vert}}

\renewcommand{\norm}[1]{\lVert#1\rVert}
\newcommand{\Norm}[1]{\left\lVert#1\right\rVert}
\newcommand{\bigNorm}[1]{\bigl\lVert#1\bigr\rVert}

\newcommand{\abs}[1]{\left\lvert#1\right\rvert}
\newcommand{\Abs}[1]{\left\lvert#1\right\rvert}
\newcommand{\bigAbs}[1]{\bigl\lvert#1\bigr\rvert}
\newcommand{\BigAbs}[1]{\Bigl\lvert#1\Bigr\rvert}

\icmltitlerunning{Doubly robust off-policy evaluation with shrinkage}

\begin{document}

\twocolumn[
\icmltitle{Doubly robust off-policy evaluation with shrinkage}



\icmlsetsymbol{equal}{*}

\begin{icmlauthorlist}
\icmlauthor{Yi Su}{cornell}
\icmlauthor{Maria Dimakopoulou}{netflix}
\icmlauthor{Akshay Krishnamurthy}{msr}
\icmlauthor{Miroslav Dud\'ik}{msr}
\end{icmlauthorlist}

\icmlaffiliation{cornell}{Cornell University, Ithaca, NY}
\icmlaffiliation{netflix}{Netflix, Los Gatos, CA}
\icmlaffiliation{msr}{Microsoft Research, New York, NY}

\icmlcorrespondingauthor{Yi Su}{ys756@cornell.edu}

\icmlkeywords{Off-policy evaluation}

\vskip 0.3in
]



\printAffiliationsAndNotice{}  

\begin{abstract}

We propose a new framework for designing estimators for off-policy
evaluation in contextual bandits. Our approach is based on the
asymptotically optimal doubly robust estimator, but we shrink the
importance weights to minimize a bound on the mean squared error, which
results in a better bias-variance tradeoff in finite samples. We use
this optimization-based framework to obtain three estimators:
(a) a weight-clipping estimator, (b) a new weight-shrinkage estimator,
and (c) the first shrinkage-based estimator for combinatorial action
sets. Extensive experiments in both standard and combinatorial bandit
benchmark problems show that our estimators are highly adaptive and
typically outperform state-of-the-art methods.

\end{abstract}

\section{Introduction}
\label{sec:intro}

Many real-world applications, ranging from online
news recommendation~\citep{li2011unbiased}, advertising~\citep{bottou2013counterfactual}, and search engines~\citep{li2015counterfactual} to personalized healthcare~\citep{zhou2017residual}, are
naturally modeled by the \emph{contextual bandit} protocol~\citep{langford2008epoch}, where
a learner repeatedly observes a context, takes an action, and accrues
reward.
In news recommendation, the context is any information about the user,
such as history of past visits, the action is the recommended article,
and the reward could indicate the user's click on the
article. The goal is to maximize the reward, but the learner can only
observe the reward for chosen actions, and not for the others.

We study a fundamental problem in contextual bandits
known as \emph{off-policy evaluation}, where the goal is to use the
data gathered by a past algorithm, known as the \emph{logging policy}, to estimate
the average reward of a new algorithm, known as the \emph{target policy}.
High-quality off-policy estimates help avoid costly A/B testing
and can also be used as subroutines for optimizing a policy~\citep{dudik2011doubly}.

The most accurate approaches to off-policy evaluation are variants of
\emph{doubly robust} (DR) estimators~\citep{robins1995semiparametric,
  bang2005doubly, dudik2011doubly}. DR estimation begins by fitting a
regression model to predict rewards as a function of context and
action.  The fitted model can be used to impute unobserved rewards of
the target policy on the training data, but such a direct estimate is
typically biased. Instead, DR adds a correction term obtained by
importance weighting the difference between observed rewards and
predicted rewards. The resulting approach is unbiased, and it is
asymptotically optimal under weaker assumptions than other
methods~\citep{rothe2016value}.  However, its finite-sample variance
can still be quite high when importance weights (also known as inverse
propensity scores) are large.
Therefore, several works have developed variants of DR that clip or
remove large importance weights. Although weight clipping incurs some
bias, it substantially decreases the variance and can yield a lower
mean squared
error~\citep{bembom2008data,bottou2013counterfactual,wang2017optimal,su2018cab}.
These works motivate weight shrinkage as a heuristic for trading off bias and variance, but they do not provide insight into
when and how these different methods should be used.


In this paper, we ask: \emph{What are the systematic strategies for
shrinking importance weights?} We seek to answer this question without making
strong assumptions about the quality of the reward predictor, but we
would like to adapt to its quality. We make the following
contributions:
\begin{itemize}
\setlength\itemsep{0.2em}
\item
We derive a general framework for shrinking the
importance weights by optimizing a sharp bound on the mean
squared error (MSE). We use two bounding techniques. The
first is agnostic to the quality of the reward estimator and yields
\emph{pessimistic shrinkage} estimators. The second incorporates
the quality of the reward predictor and yields \emph{optimistic shrinkage} estimators.\looseness=-1

\item
We provide theoretical justification for the standard practice of weight clipping
by showing that it corresponds to pessimistic shrinkage.

\item
Using optimistic shrinkage, we derive new estimators, which
are also applicable to \emph{combinatorial actions},
arising, for example, when a news portal is recommending not just a single article, but
a list of articles
\citep{cesa2012combinatorial,swaminathan2017off}.

\end{itemize}

Apart from the conceptual and theoretical contributions above, we also carry
out an extensive empirical evaluation. For atomic (i.e., non-combinatorial) actions,
we consider 108 experimental conditions derived from 9
real-world data sets and covering a range of data set sizes, feature dimensions,
policy overlap (i.e., the magnitude of importance weights), and quality
of reward estimators. For combinatorial actions, we consider a standard learning-to-rank
data set and vary the quality of reward estimators. In all instances, we demonstrate
the efficacy of our shrinkage approach.
Via extensive ablation studies, we also identify a robust
configuration of our shrinkage approach that we recommend as a
practical choice.\looseness=-1

\textbf{Comparison with related work.}
Off-policy estimation is studied in observational settings under the name \emph{average treatment
effect} (ATE) estimation, with many results on asymptotically optimal estimators~\citep{hahn1998role,hirano2003efficient,imbens2007mean,rothe2016value}, but only few
that optimize MSE in finite samples. Most notably,
\citet{Kallus17,Kallus18} develops the \emph{kernel optimal matching} (KOM) approach that
adjusts importance weights by optimizing MSE under
smoothness (or parametric) assumptions on the reward function.
This method is reminiscent of direct modeling, whose bias can be bounded under
smoothness
assumptions, but whose performance deteriorates if these
assumptions are violated. In contrast, we optimize importance weights with essentially no modeling assumptions.
Another difference is that KOM runs in time that is
super-linear in the data set size, which prevents its use with large data sets,
whereas our approach requires a single pass through the data and readily applies to large-scale scenarios.\looseness=-1

Several recent works study how to improve DR estimators
under similar assumptions as we make here~\cite{wang2017optimal,farajtabar2018more,su2018cab},
focusing either on weight shrinkage or on training of the reward predictor.
However, to our knowledge, we are the first to provide a
detailed theoretical and empirical investigation of the interplay between these two design
components. For example, in \pref{tab:ablation-eta}, we show that the \emph{more robust doubly robust} (MRDR)
approach for training of the reward predictor~\citep{farajtabar2018more} performs poorly in combination with
weight shrinkage. More generally, different estimators may require different reward predictors. This specific finding has practical
implications that are missing in prior work.\looseness=-1

\section{Setup}

We consider the \emph{contextual bandits} protocol, where a decision maker interacts with the environment by
repeatedly observing a \emph{context} $x\in\Xcal$, choosing
an \emph{action} $a \in \Acal$, and observing a \emph{reward} $r \in [0,1]$.
The context space $\Xcal$ can be uncountably large, but we assume that the
action space $\Acal$ is finite. In the news recommendation example, $x$ describes
the history of past visits of a given user, $a$ is a recommended article,
and $r$ equals one if the user clicks on the article and zero otherwise.
We assume that contexts are
sampled \emph{i.i.d.}\ from some distribution $D(x)$ and rewards are sampled
from some conditional distribution $D(r\given x,a)$.
We write $\eta(x,a) \defeq \EE\sbr{r \given x,a}$
for the expected reward, conditioned on a given context and action.


The behavior of a decision maker is formalized as a conditional distribution $\pi(a\given x)$
over actions given contexts, referred to as a \emph{policy}.
We also write $\pi(x,a,r)\defeq D(x)\pi(a\given x)D(r\given x,a)$ for the joint distribution over context-action-reward
triples when actions are selected by the policy $\pi$. The expected reward of a policy $\pi$,
called the \emph{value} of $\pi$, is denoted as
$V(\pi) \defeq \EE_{(x,a,r)\sim \pi}[r]$.

In the off-policy evaluation problem, we are given a dataset
$\{(x_i,a_i,r_i)\}_{i=1}^n \sim \mu$
consisting of context-action-reward triples collected by some
\emph{logging policy} $\mu$, and we would like to estimate the value of a \emph{target policy} $\pi$.
The quality of an estimator $\hat{V}(\pi)$ is measured by the \emph{mean
  squared error}
\begin{align*}
\mse\bigParens{\hat{V}(\pi)} \defeq \EE\BigBracks{ \bigParens{\hat{V}(\pi) - V(\pi)}^2},
\end{align*}
where the expectation is with respect to the data generation
process. In analyzing the error of an estimator, we rely on the
decomposition of MSE into the bias and variance terms:
\begin{align*}
 \mse\bigParens{\hat{V}(\pi)} &= \bias\bigParens{\hat{V}(\pi)}^2+\Var\bigBracks{\hat{V}(\pi)},\\
 \bias\bigParens{\hat{V}(\pi)} &\defeq \BigAbs{\EE\bigBracks{\hat{V}(\pi)-V(\pi)}}.
\end{align*}
We consider three standard approaches for off-policy evaluation. The
first two are
\emph{direct modeling} (DM) and \emph{inverse propensity scoring} (IPS). In DM,
we train a reward predictor $\heta: \Xcal \times \Acal \to [0,1]$ and use it to impute rewards. In IPS, we
simply reweight the data. The two estimators are:
\begin{align*}
\vdm(\pi;\heta) &\defeq \frac{1}{n}\sum_{i=1}^n \sum_{a\in\Acal} \pi(a \given x_i)\heta(x_i,a),
\\
\vips(\pi) &\defeq \frac{1}{n}\sum_{i=1}^n \frac{\pi(a_i\given x_i)}{\mu(a_i \given x_i)} r_i.
\end{align*}
Let $w(x,a) \defeq \pi(a\given x)/\mu(a \given
x)$ denote the \emph{importance weight}. We make a standard
assumption that $\pi$ is absolutely continuous with respect to $\mu$,
meaning that $\mu(a\given x)>0$ whenever $\pi(a\given x)>0$. This
ensures that the importance weights
are well defined and
$\smash{\vips(\pi)}$ is an unbiased estimator of
$V(\pi)$. If there is a substantial mismatch between $\pi$ and
$\mu$, then the importance weights will be large and $\smash{\vips(\pi)}$ will have large variance.
On the other hand, given any fixed reward
predictor $\heta$ (fit on a separate dataset), $\smash{\vdm(\pi)}$ has
low variance,
but it can be biased due to approximation errors in fitting
$\heta$.

The third approach, called the \emph{doubly robust} (DR) estimator, combines DM
and IPS:
\begin{align}
\notag
&\vdr(\pi;\heta)
\\
&\quad{}\defeq
\vdm(\pi;\heta)
+ \smash[t]{\frac{1}{n}\sum_{i=1}^{n}w(x_i,a_i)\bigParens{r_i - \heta(x_i,a_i)}}.
\label{eq:dr}
\end{align}
The DR estimator applies IPS to a shifted reward, using $\heta$ as a control variate
to decrease the variance of IPS, while preserving its unbiasedness.
DR is asymptotically optimal, as long as it is possible to derive sufficiently good reward predictors $\heta$
given enough data
~\citep{rothe2016value}.

However, even when the reward predictor $\heta$ is perfect,
stochasticity in the rewards may cause the terms $r_i-\heta(x_i,a_i)$, appearing
in the DR estimator, to be far from
zero. Multiplied by large importance weights $w(x_i,a_i)$, these terms yield large variance for DR in comparison with DM.
As mentioned in Section~\ref{sec:intro}, several approaches seek a more favorable bias--variance trade-off by shrinking the importance weights.
Our work also seeks to systematically replace the weights $w(x_i,a_i)$ with new weights $\hw(x_i,a_i)$ to
bring the variance of DR closer to that of DM.

In practice, $\heta$ is biased due to approximation errors,
so in this paper we make no assumptions about its quality. At the same time, we would like to make sure that our estimators
can adapt to high-quality $\heta$ if it is available. To motivate our adaptive estimator, we assume
that $\heta$ is trained via weighted least squares regression on a separate dataset than used in $\smash{\vdr}$. That is, for a dataset
$\smash{\{(x_j,a_j,r_j)\}_{j=1}^{m} \sim \mu}$, we consider a weighting function $z:\Xcal\times\Acal\to\RR^+$
and solve
\begin{align}
\heta \defeq \argmin_{f \in \Fcal} \frac{1}{m}\sum_{j=1}^{m} z(x_j,a_j)\bigParens{f(x_j,a_j) - r_j}^2,\label{eq:sq_opt}
\end{align}
where $\Fcal$ is some function class of reward predictors. Natural
choices of the weighting function $z$, explored in our experiments,
include $z(x,a)=1$, $z(x,a)=w(x,a)$ and $z(x,a)=w^2(x,a)$. We stress
that the assumption on how we fit $\heta$ only serves to guide our
derivations, but we make no specific assumptions about its quality. In
particular, we do not assume that $\Fcal$ contains a good
approximation of $\eta$.\looseness=-1

\section{Our Approach: DR with Shrinkage}
\label{sec:DRs}

Our approach replaces the importance-weight mapping $w: \Xcal
\times \Acal \to \RR^+$ in the DR estimator~\pref{eq:dr} with a new weight mapping $\hw: \Xcal
\times \Acal \to \RR^+$
found by directly optimizing sharp bounds on the MSE.
%
The resulting estimator, which we call  the \emph{doubly robust estimator with shrinkage} (\drs) thus depends
on both the reward predictor $\heta$ and the weight mapping $\hw$:
\begin{align}
\notag
&\vdrs(\pi;\heta,\hw)
\\
&\quad{}\defeq
\vdm(\pi;\heta)
+ \smash[t]{\frac{1}{n}\sum_{i=1}^{n}\hw(x_i,a_i)\bigParens{r_i - \heta(x_i,a_i)}}.
\label{eq:drs}
\end{align}
We assume that $0\le\hw\le w$, justifying the terminology ``shrinkage''.
For a fixed choice of $\pi$ and $\heta$, we will seek the mapping $\hw$ that minimizes the MSE of $\vdrs(\pi;\heta,\hw)$, which we simply denote as $\mse(\hw)$.
We similarly write $\bias(\hw)$ and $\Var(\hw)$ for the bias and variance of this estimator. 

We treat $\hw$ as the optimization variable and consider two upper bounds on MSE: an optimistic one and a pessimistic one. In both cases, we separately bound $\bias(\hw)$ and $\Var(\hw)$. To bound the bias, we use the following expression, derived from the fact that $\smash{\vdrs}$ is unbiased when $\hw=w$:
\begin{align}
\notag
\bias(\hw)
&=\BigAbs{\EE\bigBracks{\vdrs(\pi;\heta,\hw)}-\EE\bigBracks{\vdrs(\pi;\heta,w)}}
\\
\label{eq:bias}
&=\BigAbs{\E_{\mu}\bigBracks{\bigParens{\hw(x,a) - w(x,a)}\bigParens{r - \heta(x,a)}}}.
\mkern-1mu
\end{align}
To bound the variance, we rely on the following proposition, which states that it suffices to focus on the second moment of the terms $\hw(x_i,a_i)\bigParens{r_i-\heta(x_i,a_i)}$:
\begin{proposition}
\label{prop:variance}
If $0\le\hw\le w$
then
\[
\Abs{\Var(\hat{w})-\frac{1}{n}\E_{\mu}\BigBracks{\hw^2(x,a)\bigParens{r-\heta(x,a)}^2}}\le\frac{1}{n}.
\]
\end{proposition}
See appendix for the proof of \pref{prop:variance} (as well as other mathematical statements from this paper).

We derive estimators for two different regimes depending on the
quality of the reward predictor~$\heta$. Since we do not know the
quality of $\heta$ a priori, in Section~\ref{sec:modelsel} we derive a model
selection procedure to select between these two estimators.\looseness=-1



\subsection{DR with Optimistic Shrinkage}
\label{sec:opt}

Our first family of estimators is based on an optimistic MSE bound, which
adapts to the quality of~$\heta$, and which we expect to be tighter
when $\heta$ is more accurate.
Recall that $\heta$ is trained to minimize weighted square loss with respect
to some weighting function $z$, which we denote as
\[
   L(\heta)\defeq\E_{\mu}\bigBracks{z(x,a)\bigParens{r - \heta(x,a)}^2}.
\]
%
The loss $L(\heta)$ quantifies the quality of $\heta$. We use it to bound
the bias by applying the Cauchy--Schwarz inequality to~\eqref{eq:bias}:
\begin{multline}
\label{eq:bias:opt}
\bias(\hat{w})
\le
\sqrt{\Emu{\tfrac{1}{z(x,a)}\bigParens{\hw(x,a) - w(x,a)}^2}}
\\
{}\cdot
\sqrt{L(\heta)}.
\end{multline}
To bound the variance, we invoke~\pref{prop:variance} and
focus on bounding
the quantity $\E_{\mu}\bigBracks{\hw^2(r-\heta)^2}$:
\begin{align}
\notag
&\Emu{\hw^2(x,a)\bigParens{r - \heta(x,a)}^2}\\
&\quad{}\le
\sqrt{\Emu{\tfrac{1}{z(x,a)}\hw^4(x,a)}}
\sqrt{\Emu{z(x,a)\bigParens{r - \heta(x,a)}^4}} \notag
\\
\label{eq:var:opt}
&\quad{}\le
\sqrt{\Emu{\tfrac{w^2(x,a)}{z(x,a)}\hw^2(x,a)}}
\sqrt{L(\heta)},
\end{align}
where the first inequality follows by the Cauchy-Schwarz inequality,
and the second from the fact that $\hw^2(x,a)\le w^2(x,a)$ and $\abs{r-\heta(x,a)}\le 1$.

Combining the bounds~\eqref{eq:bias:opt} and~\eqref{eq:var:opt} with~\pref{prop:variance} yields the following bound on $\mse(\hw)$:
\begin{align*}
\mse(\hw)
& \le
  \Emu{\tfrac{1}{z(x,a)}\bigParens{\hw(x,a) - w(x,a)}^2} L(\heta)
\\
& \qquad{} +
  \sqrt{\Emu{\tfrac{w^2(x,a)}{z(x,a)}\hw^2(x,a)}}
  \sqrt{L(\heta)}
  +
  \frac1n.
\end{align*}
A direct minimization of this bound appears to be a high dimensional
optimization problem. Instead of minimizing the bound directly, we
note that it is a strictly increasing function of the two expectations
that appear in it. Thus, its minimizer must be on the Pareto front
with respect to the two expectations, meaning that for some choice of
$\lambda\in[0,\infty]$, it can be obtained by minimizing
\[
\lambda\Emu{\tfrac{1}{z(x,a)}\bigParens{\hw(x,a) - w(x,a)}^2}
+
\Emu{\tfrac{w^2(x,a)}{z(x,a)}\hw^2(x,a)}
\]
with respect to $\hw$.
%
%
This objective decomposes across contexts and actions. Taking the derivative with respect to $\hw(x,a)$ and setting it to
zero yields the solution
\begin{align*}
\hw_{\opt,\lambda}(x,a) = \frac{\lambda}{w^2(x,a) +\lambda} w(x,a),
\end{align*}
where ``\opt'' above is a mnemonic for optimistic shrinkage.
%
%
We refer to the \drs estimator with $\hw=\hw_{\opt,\lambda}$ as the
\emph{doubly robust estimator with optimistic shrinkage} (\ltwo)
and denote it by $\smash{\vdrso(\pi;\heta,\lambda)}$.
Note that this
estimator does not depend on $z$, although
it was included in the optimization objective. When $\lambda=0$, we have
$\hw(x,a)=0$ corresponding to
DM. As $\lambda\to\infty$, the weights increase and in the limit
become equal to $w(x,a)$, corresponding to standard DR.


\subsection{DR with Pessimistic Shrinkage}

Our second estimator family makes no assumptions on the quality of $\heta$
beyond the range bound $\heta(x,a)\in[0,1]$, which implies $\abs{\heta(x,a)-r}\le 1$
and yields the bounds
%
\begin{align}
\label{eq:pes}
& \bias(\hw) \leq \E_{\mu}\bigBracks{\abs{\hw(x,a) - w(x,a)}},\\
& \E_{\mu} \bigBracks{\hat{w}(x,a)^2(r-\heta(x,a))^2} \leq \E_{\mu}\bigBracks{\hat{w}(x,a)^2}.
\end{align}
%
As before, we do not optimize the resulting MSE bound directly and instead solve for
the Pareto front points parameterized by $\lambda\in[0,\infty]$ (we scale $\lambda$ by a factor
of two to obtain the solution that more cleanly
matches the clipping estimator):
\[
  \Minimize_{\hw}\;
2\lambda\E_{\mu}\bigBracks{\abs{\hw(x,a) - w(x,a)}}
+
\E_{\mu}\bigBracks{\hat{w}(x,a)^2}.
\]
%
The objective again decomposes across
context-action pairs, yielding the solution
\begin{align*}
\hat{w}_{\pes,\lambda}(x,a) = \min\set{\lambda,\, w(x,a)},
\end{align*}
which recovers (and justifies) existing weight-clipping approaches~\cite{kang2007demystifying,strehl2010learning,su2018cab} (see~\pref{app:examples} for detailed calculations). We
refer to the resulting estimator as $\smash{\vdrsp(\pi;\heta,\lambda)}$, for
\emph{doubly robust with pessimistic shrinkage}.
Similarly to optimistic shrinkage, we
recover DM for $\lambda=0$, and DR as $\lambda\to\infty$.

\section{Shrinkage for Combinatorial Actions}
\label{sec:slates}

We showcase the generality of our optimization-based
approach by deriving a shrinkage estimator for \emph{combinatorial actions}
(also called \emph{slates}), which arise, for example,
when recommending a ranked list of items.

In contextual combinatorial bandits,
the actions are represented as
vectors $\ab \in \RR^d$ for some dimension $d$
and the action space $\Acal\subseteq\RR^d$ is typically exponentially large in $d$.

\begin{example}[Ranking and NDCG]
\label{ex:slate}
Consider the task of recommending a ranked list of items such as images or web pages.
The context $x$ is the query submitted by a user
together with a user profile. The action $\ab$ represents a ranked list of $\ell$ items out
of $m$.
The list $(i_1,\dotsc,i_\ell)$, where $i_j\in\set{1,\dotsc,m}$, is encoded into an action vector
$\ab\in\set{0,1}^{\ell m}$ via \emph{$\ell$-hot encoding}, i.e., we split $\ab$ into
$\ell$ blocks of size $m$, and in the block $j$ we set the $i_j$-th coordinate to 1 and all others to 0.
As a reward
we use a standard information-retrieval metric called the \emph{normalized discounted cumulative gain},
defined as
$\ndcg(x,\ab) \defeq \dcg(x,\ab)/\dcgs(x)$,
where
$\dcg(x,\ab) \defeq \sum_{j=1}^\ell \tfrac{2^{\rel(x,i_j)}-1}{\log_2(j+1)}$,
$\dcgs(x) \defeq \max_{\ab'} \dcg(x,\ab')$,
and $\rel(x,i)$ is some intrinsic measure of item relevance.
(See, e.g., \citealp{swaminathan2017off}.)\looseness=-1
\end{example}

Standard importance weighting techniques, such as DR and IPS, can fail dramatically in the
combinatorial setting, because their variance
scales linearly with the size of $\Acal$, which
is typically exponential in $d$. However, if the expected reward is linear in $\ab$,
i.e., $\eta(x,\ab)=\etab(x)^\top\ab$ for some (unknown) function
$\etab:\Xcal\to\RR^d$, then it is possible to achieve variance polynomial in $d$
using the \emph{pseudo-inverse} estimator of~\citet{swaminathan2017off}. Given a
reward predictor $\hetab:\Xcal \to \RR^d$,
it is also possible to obtain the DR variant of this estimator (\mbox{DR-PI}):
%
%
%
\begin{align}
\label{eq:drpi}
\vdrpi(\pi;\hetab)
\defeq \frac{1}{n}\sum_{i=1}^n\hetab_i^\top \qb_{\pi,x_i}\! + \wb_i^\top \ab_i (r_i - \hetab_i^\top \ab_i),
\end{align}
%
where $\hetab_i \defeq \hetab(x_i)$, $\qb_{\pi,x_i}
\defeq \E_\pi[\ab \given x_i]$, $\wb_i \defeq
\Gamma_{\mu,x_i}^\dagger \qb_{\pi,x_i}$, $\Gamma_{\mu,x_i} \defeq
\E_\mu[\ab \ab^\top \given x_i]$, and $\dagger$ denotes the
matrix pseudo-inverse. The vector $\wb_i$ plays the role of the
importance weight, while the first term corresponds to the direct
modeling approach. \citet{swaminathan2017off}
establish that this estimator is unbiased when $\eta(x,\ab)$ is
linear in~$\ab$ and
$
  \linspan\bigParens{\supp\pi(\cdot\given x)}
  \subseteq
  \linspan\bigParens{\supp\mu(\cdot\given x)}
$,
which is a linear relaxation of absolute continuity. Note that
NDCG in \pref{ex:slate} satisfies the linearity assumption.




We next derive the shrunk variant of DR-PI, following the optimistic
bounding technique from Section~\ref{sec:opt}.
A formal difference
is that we seek a vector-valued map $\hat{\wb}:\Xcal\to\RR^d$.
Since $\wb(x)^\top\ab$ can be negative,
we formalize the shrinkage property
as $\hat{\wb}(x)^\top\ab=c(x,\ab)\wb(x)^\top\ab$ for some $c(x,\ab)\in[0,1]$.
Also, analogously to non-combinatorial
setup, we assume that $\hetab(x)^\top\ab \in [0,1]$ for all $\ab$.
Now
all the steps from Section~\ref{sec:opt}, except for \pref{prop:variance}
(to which we return below),
go through under substitution
$w(x,\ab)=\wb(x)^\top\ab$,
$\hw(x,\ab)=\hat{\wb}(x)^\top\ab$, and
$\heta(x,\ab)=\hat{\etab}(x)^\top\ab$. The resulting (optimistic) shrinkage
estimator takes form
\begin{align}
\notag
&
\vdrsopi(\pi;\hetab,\lambda)
\\
\label{eq:drsopi}
&\quad{}\defeq
\frac{1}{n}\sum_{i=1}^n\hetab_i^\top \qb_{\pi,x_i}\!
+ \frac{\lambda\wb_i^\top\ab_i}{\lambda+(\wb_i^\top\ab_i)^2}(r_i - \hetab_i^\top\ab_i).
\end{align}
The detailed derivation is in Appendix~\ref{app:slates}. To our knowledge
this is the first weight-shrinkage estimator for contextual combinatorial bandits.

To finish the section, we derive a combinatorial variant
of \pref{prop:variance}, establishing a tight, but simple-to-optimize
proxy for the variance of a DR-PI.
This requires an additional assumption that for each $x$,
the logging policy is supported on a linearly independent
set of actions $\smash{\Bcal_x}\subseteq\Acal$; this requirement is typically easy to
satisfy in practice (see, e.g., \pref{sec:exp-slates}). We write $B_x\in\RR^{d\times\card{\Bcal_x}}$
for the matrix with columns $\ab\in\smash{\Bcal_x}$, and $\smash{\vb_{\pi,x}}$ for the unique vector such that $B_x\vb_{\pi,x}=\qb_{\pi,x}$. Finally, let $\Var(\hwb)$ denote the variance
of a DR-PI estimator with the shrunk weight map $\hwb$.
\begin{proposition}
\label{prop:variance_pi}
Assume that $\mu(\cdot\given x)$ is supported on a linearly independent
set of actions for every $x$. If $\hat{\wb}(x)^\top\ab=c(x,\ab)\wb(x)^\top\ab$ for some $c(x,\ab)\in[0,1]$,
then
\begin{align*}
\Abs{\Var(\hwb) - \frac{1}{n}\E_{\mu}\BigBracks{(\hwb^\top\ab)^2\bigParens{r-\hetab^\top\ab}^2}}
\leq \frac1n\E_x[\norm{\vb_{\pi,x}}_1^2].
\end{align*}
\end{proposition}
Note that the quantity $\norm{\vb_{\pi,x}}_1$ on the right-hand side only depends on
the set $\Bcal_x$, but not on the probabilities with which $\mu$ chooses $\ab\in\Bcal_x$.
Non-combinatorial setting of \pref{sec:DRs} is a special case of the
linearly independent setting, where $d=\card{\Acal}$ and actions are represented by
standard basis vectors. In this case, $\norm{\vb_{\pi,x}}_1=1$ and we recover
\pref{prop:variance}.
We can always select $\Bcal_x$ to be
an (approximate) \emph{barycentric spanner}
and achieve $\norm{\vb_{\pi,x}}_1=O(d)$~\citep{AwerbuchKl08,DaniHaKa08}.\looseness=-1

\section{Model Selection}
\label{sec:modelsel}

All of our shrinkage estimators have hyperparameters which we condense
into a tuple $\theta$. For example $\theta = (\heta,\opt,\lambda)$
denotes that we are using a reward predictor $\heta$ and
optimistic shrinkage with the parameter $\lambda$. To select among
these hyperparameters, we propose and analyze a simple model selection
procedure.


Let $\vtheta$ denote the estimator parameterized by $\theta$.
We consider the procedure that estimates
the variance of $\vtheta$ by sample variance $\smash{\hvar(\theta)}$,
and bounds the bias of $\vtheta$ by a
data-dependent
upper bound $\hbias(\theta)$. The only
requirement is that for all~$\theta$, $\bias(\theta) \leq
\hbias(\theta)$ (with high probability), and that $\hbias(\theta)=0$
whenever $\bias(\theta)=0$; this holds for both bias bounds from \pref{sec:DRs},
as they become zero when $\hw=w$.  Now, to
choose $\theta$ from a set of hyperparameters $\Theta$, we
optimize the estimate of the MSE:
\begin{align*}
\hat{\theta} \gets
\Minimize_{\theta \in \Theta}\;\hbias(\theta)^2 + \hvar(\theta).
\end{align*}
The next theorem shows that this procedure always compares favorably with
all the unbiased estimators included in $\Theta$, up to an asymptotically negligible
term $O(n^{-3/2})$. In particular, the procedure is asymptotically optimal
whenever $\Theta$ includes a standard (non-shrunk) DR.

\begin{theorem}
\label{thm:model_selection}
Let $\Theta$ be a finite set of hyperparameter values and let
$\Theta_0\defeq\set{\theta\in\Theta:\:\bias(\theta)=0}$ denote the
subset of unbiased estimators. Assume that with probability
$1-\delta/2$ we have $\bias(\theta) \leq \hbias(\theta)$ for all
$\theta \in \Theta$. Then there exists a universal constant $C$ such
that with probability at least $1-\delta$ we have
\begin{align*}
\smash[b]{\mse(\hat{\theta}) \leq \min_{\theta_0 \in \Theta_0} \mse(\theta_0) + C\log (|\Theta|/\delta)/n^{3/2}.}
\end{align*}
\end{theorem}

There are many strategies to construct data-dependent bias bounds
with the required properties. The three bounds in our
experiments take form of sample averages that approximate expectations in: (i)
the expression for the bias given in~\eqref{eq:bias}, (ii) the
optimistic bias bound in~\eqref{eq:bias:opt}, and (iii) the pessimistic bias
bound in~\eqref{eq:pes}. In our theory, these estimates need to be adjusted
to obtain high-probability
confidence bounds.
In our experiments, we evaluate both the basic estimates and
adjusted variants where we add twice the standard error.\looseness=-1


Our model selection procedure is related to
MAGIC \citep{thomas2016data} as well as the procedure for the SWITCH
estimator~\citep{wang2017optimal}. Unlike MAGIC, we pick a
single hyperparameter value $\theta$ rather than aggregating several, and
we use different bias and variance estimates. SWITCH uses our
pessimistic bias bound~\eqref{eq:pes},
but with no theoretical justification. We use two additional
bounding strategies, which are empirically shown to help,
and provide theoretical justification in the form of an oracle inequality.

\section{Experiments}
\label{sec:experiments}

We evaluate our new estimators on the tasks of off-policy evaluation and off-policy learning and
compare their performance with previous estimators. Our secondary goal is to identify
the configuration of the shrinkage estimator that is most robust for use in practice.


\begin{table*}%
{\small%
\vspace{-0.7\baselineskip}%
\newcommand{\myheight}{1.7in}%
\addtolength{\tabcolsep}{-1pt}%
\begin{minipage}[t][\myheight]{.23\linewidth}%
\caption{Policy parameters used in the experiments.}
\label{tab:policies}
\vfill
\centering%
\begin{tabular}{|c|c|c|c|}
\hline
                                & base                                         & $\alpha$          & $\beta$ \\ \hline
target                & $\pi_{1,\textrm{det}}$ & 0.9                            & 0                    \\ \hline
\multirow{5}{*}{$\!$logging$\!$} & $\pi_{1,\textrm{det}}$ & 0.7                            & 0.2                  \\
                     & $\pi_{1,\textrm{det}}$ & 0.5                            & 0.2                  \\
                     & ---                    & $1/k$                  & 0                    \\
                     & $\pi_{2,\textrm{det}}$ & 0.3                            & 0.2                  \\
                     & $\pi_{2,\textrm{det}}$ & 0.5                            & 0.2                  \\
                     & $\pi_{2,\textrm{det}}$ & 0.95                           & 0.1 \\ \hline
\end{tabular}

\end{minipage}
\addtolength{\tabcolsep}{1pt}%
\hfill
\addtolength{\tabcolsep}{-2.25pt}%
\begin{minipage}[t][\myheight]{0.41\linewidth}
\caption{Comparison of
  reward predictors using a fixed estimator (with oracle tuning if
  applicable); reporting the number of conditions where a regressor is
  statistically as good as the best and, in parenthesis,
  the number of conditions where it statistically dominates all
  others.}
\label{tab:ablation-eta}
\vfill
\centering%


\begin{tabular}{| l | c | c | c | c | c |}
\hline
& $\hat{\eta} \equiv 0$ & $z \equiv 1$ & $z = w$ & $z = w^2$ & MRDR\\
\hline\hline
DM & 0 (0) & 47 (23) & 45 (22) & 41 (31) & 11 (5)\\
DR & 27 (2) & 86 (9) & 90 (4) & 85 (5) & 65 (0)\\
snDR & 63 (7) & 80 (2) & 85 (8) & 69 (4) & 54 (0)\\
DRs & 23 (19) & 44 (16) & 35 (4) & 62 (35) & 18 (2)\\
\hline
\end{tabular}
\end{minipage}
\addtolength{\tabcolsep}{2.25pt}%
\hfill
\begin{minipage}[t][\myheight]{.31\linewidth}
\caption{Comparison of shrinkage types using a fixed reward
  predictor (with oracle tuning); reporting the number of conditions
  where one statistically dominates the other.}
\label{tab:ablation-shrinkage}
\vfill
\centering%
\begin{tabular}{| l | c | c |}
\hline
& \lone & \ltwo \\
\hline\hline
$\hat{\eta} \equiv 0$ & 21 & 51\\
$z \equiv 1$ & 58 & 28\\
$z = w$ & 55 & 30\\
$z = w^2$ & 55 & 29\\
\textsc{MRDR} & 49 & 29\\
\hline
\end{tabular}
\end{minipage}
}%
\end{table*}

\subsection{Non-combinatorial Setting}

\textbf{Datasets.}
Following prior work~\cite{dudik2014doubly,
wang2017optimal, farajtabar2018more, su2018cab}, we simulate bandit
feedback on 9 UCI multi-class classification datasets. This lets us
evaluate estimators in a broad range of conditions and gives us
ground-truth policy values (see~\pref{tab:datasets} in the appendix for the
dataset statistics). Each multi-class dataset with $k$ classes corresponds
to a contextual bandit problem with $k$ possible actions coinciding with
classes. We consider either \emph{deterministic rewards}, where on multiclass
example $(x,y^*)$, the action $y$ yields the reward $r=\one\braces{y=y^*}$, or
\emph{stochastic rewards} where $r=\one\braces{y=y^*}$ with probability $0.75$ and
$r=1-\one\braces{y=y^*}$ otherwise. For every dataset, we hold
out $25\%$ of the examples to measure ground truth. On the remaining
$75\%$ of the dataset,
we use logging policy $\mu$ to simulate $n$ bandit examples by sampling a context $x$
from the dataset,
sampling an action $y \sim \mu(\cdot\given x)$ and then observing
a deterministic or stochastic reward $r$. The value of $n$ varies across
experimental conditions.



\textbf{Policies.}
We use the $25\%$ held-out data to obtain logging and target policies
as follows. We first obtain two deterministic policies
$\pi_{1,\textrm{det}}$ and $\pi_{2,\textrm{det}}$ by training two logistic models on the same data, but using either the
first or second half of the features.
We obtain stochastic
policies parameterized by $(\alpha,\beta)$, following the \emph{softening} technique of \citet{farajtabar2018more}. Specifically, $\pi_{1,(\alpha,\beta)}(a \given x)
= (\alpha+\beta u)$ if $a = \pi_{1,\textrm{det}}(x)$ and
$\pi_{1,(\alpha,\beta)}(a \given x) = \frac{1-\alpha-\beta u}{k-1}$
otherwise, where $u \sim
\textrm{Unif}([-0.5,0.5])$. In off-policy evaluation
experiments, we consider a fixed target and several choices
of logging policy
(see~\pref{tab:policies}). In off-policy learning we use
$\pi_{1,(0.9,0)}$ as the logging policy.

\textbf{Reward predictors.}
We obtain reward predictors $\heta$ by training linear models via weighted least
squares with $\ell_2$
regularization.
We consider weights $z(x,a) \in\{1,
w(x,a), w^2(x,a)\}$ as well as the \emph{more robust doubly robust}, or MRDR, weight design of~\citet{farajtabar2018more}
(see~\pref{app:experiments}). In evaluation experiments, we use $1/2$ of the bandit data
to train $\heta$; in learning experiments, we use $1/3$ of the bandit data to train $\heta$. In addition to the four
trained reward predictors, we also consider $\heta\equiv 0$.
The remaining bandit data is used to calculate the value of each estimator.\looseness=-1

\textbf{Baselines.}
We include a number of estimators in our evaluation: the
direct modeling approach (DM), doubly-robust approach (DR) and its
self-normalized variant (\sndr), our approach (\drs), and the
doubly-robust version of the \switch estimator of \citet{wang2017optimal}, which also performs a form of weight
clipping.\footnote{For simplicity we call this estimator \switch, although
  Wang et al.\ call it \textsc{switch-DR}.}
Note that DR with $\hat{\eta}\equiv 0$ is identical to inverse
propensity scoring (IPS); we refer to its self-normalized variant as
\snips.
Our estimator and \switch have hyperparameters, which are selected by their
respective model selection procedures (see~\pref{app:experiments} for
details about the hyperparameter grid).

\subsubsection{Off-policy Evaluation}

We begin by evaluating different configurations of \drs via an
ablation analysis. Then we compare \drs with baseline estimators. We have a
total of $108$ experimental conditions: for each of the $9$ datasets
we use $6$ logging policies and consider stochastic or deterministic
rewards. Except for the learning curves below, we always take $n$
to be all available bandit data ($75\%$ of the overall dataset).

We measure performance with clipped MSE, $\EE\bigBracks{(\hat{V} - V(\pi))^2
  \wedge 1}$, where $\hat{V}$ is the estimator and $V(\pi)$ is the
ground truth (computed on the held-out $25\%$ of the data). We
use 500 replicates of bandit-data generation
 to estimate the MSE; statistical
comparisons are based on paired $t$-tests at significance level $0.05$. In some of our ablation
experiments, we pick the best hyperparameters against the test set on a
per-replicate basis, which we call \emph{oracle tuning} and always
call out explicitly.\looseness=-1

\textbf{Ablation analysis.}
We conduct two ablation studies:
one evaluating different reward predictors and the other evaluating
the optimistic and pessimistic shrinkage types.

In \pref{tab:ablation-eta}, for each fixed estimator type (e.g., DR) we
evaluate each reward predictor by reporting the number of conditions
where it is statistically indistinguishable from the best and the
number of conditions where it statistically dominates all other
predictors.  For \drs we use oracle tuning
for the shrinkage type and coefficient $\lambda$.
The table shows that weight shrinkage strongly influences the choice of
regressor.
For example, $z\equiv 1$ and $z=w$ are top choices for DR, but
with the inclusion of shrinkage in \drs, $z=w^2$ emerges as the
best choice.
%
In our comparison experiments below, we run each method with
its best reward predictor: DM with $z\equiv 1$, snDR with $z=w$, and
\drs and \switch with $z=w^2$. For \drs and \switch, we additionally also consider $\heta\equiv 0$,
because it allows including IPS as their special case.  Somewhat
surprisingly, in our experiments, MRDR is dominated by other reward
predictors (except for $\heta\equiv0$), and this remains true even
with a deterministic target policy (see \pref{tab:ablations_det}
in the appendix).


In \pref{tab:ablation-shrinkage}, we compare optimistic and pessimistic shrinkage when
paired with a fixed reward predictor (using oracle tuning for $\lambda$).
We report how many times each
estimator statistically dominates the other. The results suggest
that both shrinkage types are important for robust performance across
conditions, so we consider both choices going forward.\looseness=-1



\begin{figure*}[!htb]
\includegraphics[width=\textwidth]{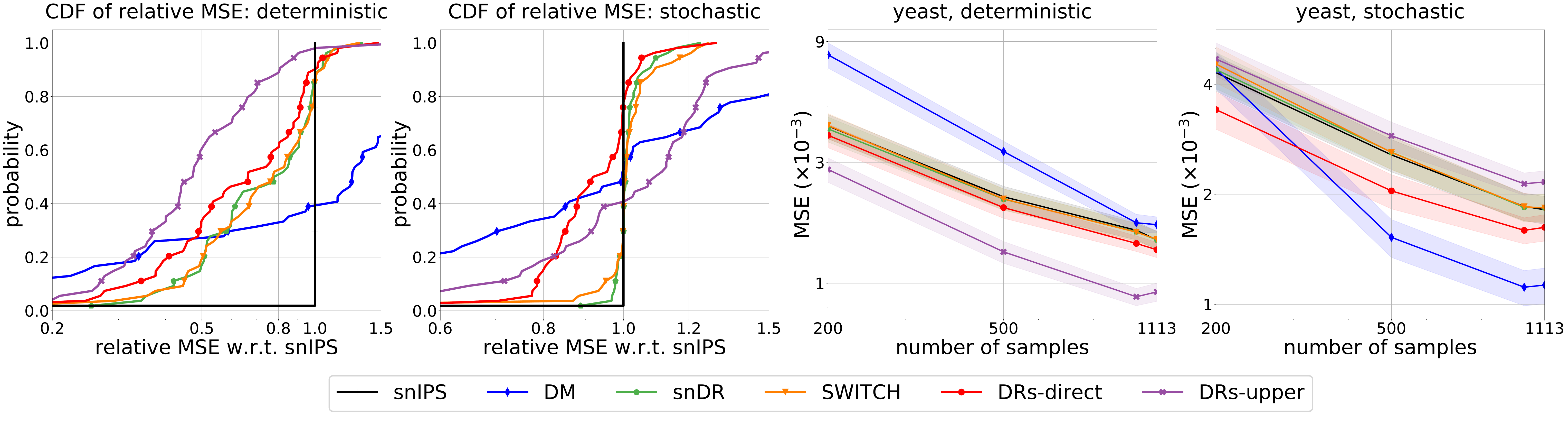}
\vspace{-0.8cm}
\caption{From left to right: (1) CDF of relative MSE w.r.t.\ snIPS for deterministic rewards, 54 conditions in total; (2) CDF of relative MSE w.r.t.\ snIPS for stochastic rewards, 54 conditions in total; (3) learning curves on \emph{yeast} dataset, using base policy $\pi_1$ with $\alpha=0.7$ and $\beta=0.2$, deterministic reward; (4) learning curves on \emph{yeast} dataset, using base policy $\pi_1$ with $\alpha=0.7$ and $\beta=0.2$, stochastic reward.}
\label{fig:atomic_fig}
\end{figure*}

\textbf{Comparisons.}
In~\pref{fig:atomic_fig} (left two plots), we compare our new estimator with the
baselines. We visualize the results by plotting the cumulative
distribution function (CDF) of the normalized MSE of each method (normalized by the
MSE of \snips) across the experimental conditions. Better performance corresponds
to CDF curves towards the top-left corner, meaning the method
achieves a lower MSE more frequently. The first plot
summarizes 54 conditions where the reward is deterministic, while the
second plot considers the 54 stochastic reward conditions. For
\drs we consider two model selection procedures outlined in \pref{sec:modelsel} that differ in
their choice of $\hbias$.
\drsdirect estimates the expectations in the expressions in Eqs.~\eqref{eq:bias},~\eqref{eq:bias:opt}, and~\eqref{eq:pes} (corresponding to the bias and bias bounds) by empirical averages and takes their pointwise minimum. \drsup adds to these estimates twice
their standard error, before taking minimum, more closely matching our
theory.
For \drs, we use the zero reward predictor and
the one trained with $z = w^2$, and we always select between both
shrinkage types. Since \switch also comes with a model selection
procedure, we use it to select between the same two reward predictors as $\drs$.\looseness=-1


In the deterministic case (the first plot), we see that \drsup has the best
aggregate performance, by a large margin. \drsdirect also has better
aggregate performance than the baselines on most of the conditions. In
the stochastic case (the second plot), \drsdirect has similarly strong performance, but
\drsup degrades considerably, suggesting this model selection scheme
is less robust to stochastic rewards.
We illustrate this phenomenon in the right two plots of~\pref{fig:atomic_fig}, plotting
the MSE as a function of the number of samples for one choice
of a logging policy and dataset, first with deterministic rewards and then
with stochastic rewards. Because of a more robust performance,
we therefore advocate for
\drsdirect as our final method.


\begin{figure*}[!htb]
\includegraphics[width=\textwidth]{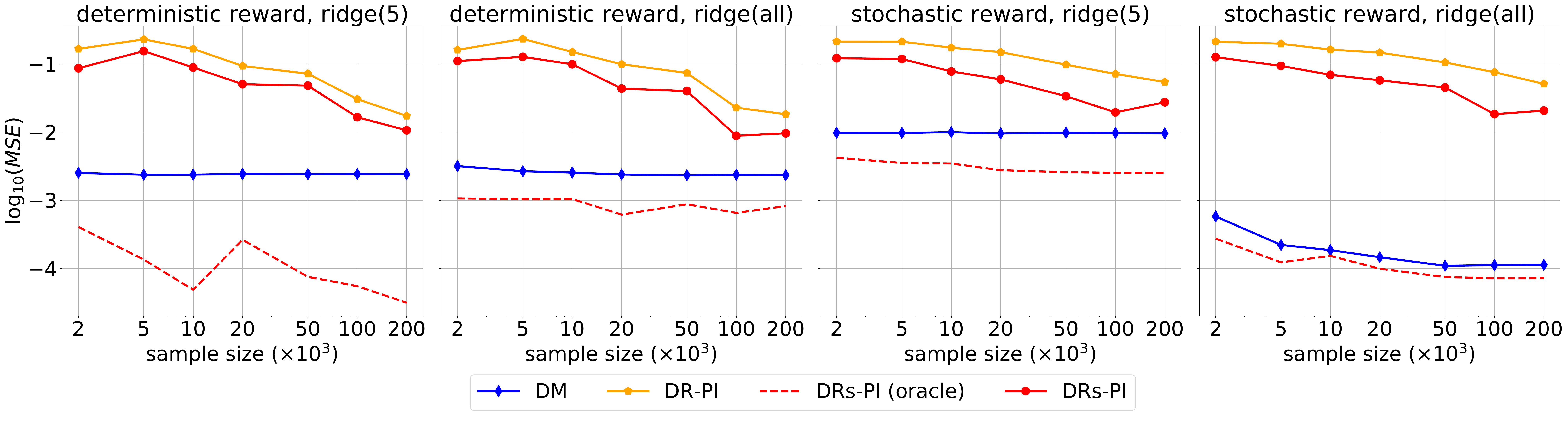}
\vspace{-0.8cm}
\caption{\emph{Off-policy evaluation with combinatorial actions.} MSE as a function of sample sizes for two reward distributions (deterministic and stochastic), and two reward predictors (ridge regression with five and all features). The MSE of DR-PI is significantly larger than DR-PIs in all cases (with $p$-value below 0.013 according to a paired $t$-test).}
\label{fig:slate}
\vspace{-0.2cm}
\end{figure*}

\subsubsection{Off-policy Learning}
\begin{figure}
\begin{center}
\includegraphics[width=\columnwidth]{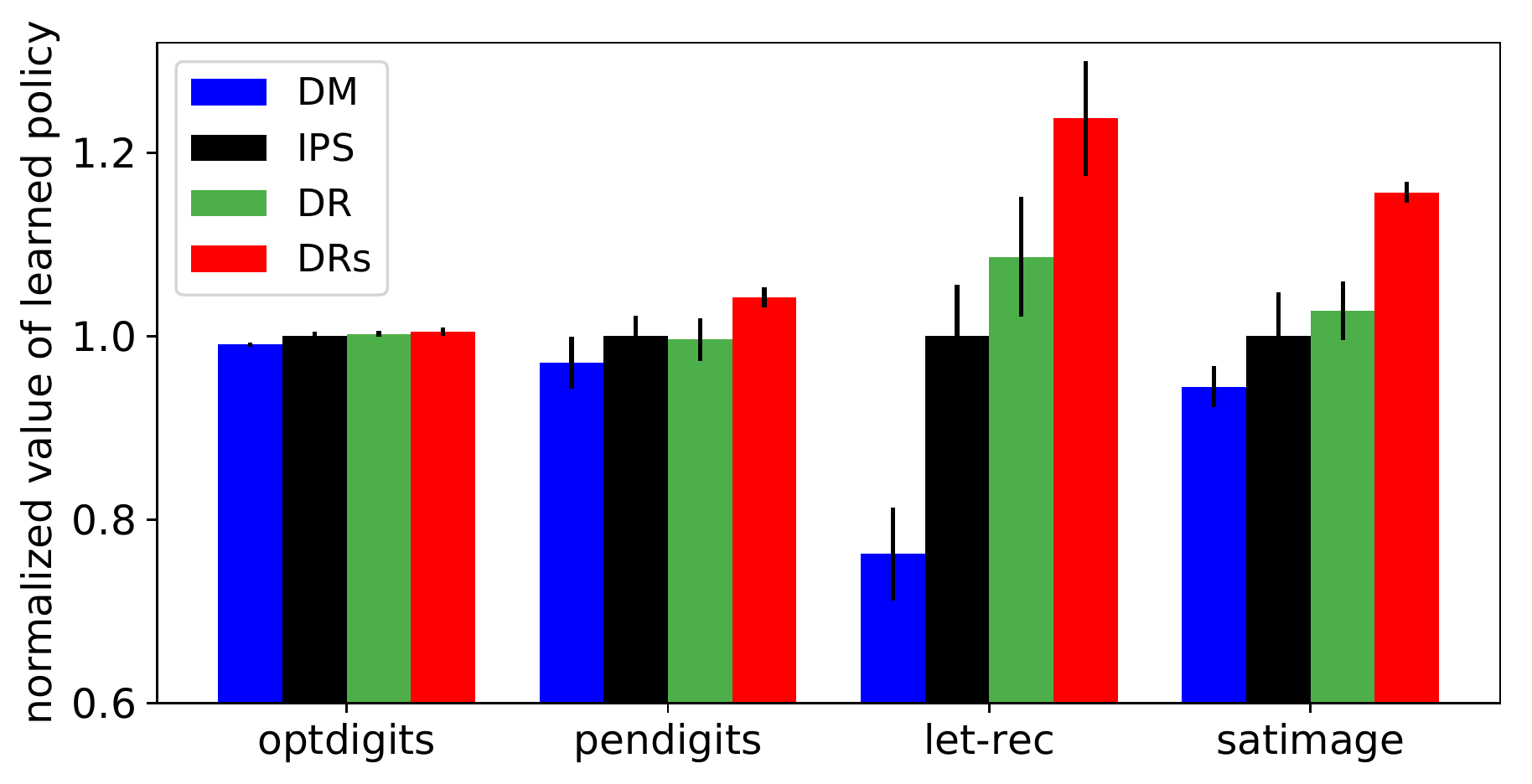}
\vspace{-0.75cm}
\caption{\emph{Off-policy learning experiments.}}
\label{fig:learning}
\vspace{-0.5cm}
\end{center}
\end{figure}
Following prior work~\cite{swaminathan2015counterfactual, swaminathan2015self,
  su2018cab}, we learn a stochastic linear policy $\pi_\ub$ where
$\pi_\ub(a \given x) \propto \exp\bigBraces{\ub^\top \fb(x,a)}$ and $\fb(x,a)$
is a featurization of context-action pairs. We solve $\ell_2$-regularized empirical risk
minimization $\hat{\ub} = \argmin_{\ub}\bigBracks{-\hat{V}(\pi_{\ub}) +
\gamma \norm{\ub}^2}$ via gradient descent, where $\smash{\hat{V}}$ is a
policy-value estimator and $\gamma>0$ is a hyperparameter. For these
experiments, we partition the data into four quarters: one
full-information segment for training the logging policy and as a test set,
and three bandit segments for (1) training reward predictors, (2)
learning the policy, and (3) hyperparameter tuning and model selection. The
logging policy is $\pi_{1,(0.9,0)}$ and since there is no fixed target
policy, we consider three reward predictors: $\hat{\eta} \equiv 0$,
and $\hat{\eta}$ trained with $z =
1/\mu(a\given x)$ and $z = 1/\mu(a\given x)^2$.\looseness=-1

In~\pref{fig:learning}, we show the performance of four methods (DM,
DR, IPS, and \drsdirect) on four of the UCI datasets. For each method,
we compute the average value of the learned policy on the test set
(averaged over 10 replicates) and report this value normalized by
the average value for IPS. For DM and DR, we select the hyperparameter $\gamma$ and reward predictor optimally in hindsight, while for DRs we use our model selection. Note that we do not compare with \switch here as
it is not amenable to gradient-based
optimization~\cite{su2018cab}. We find that off-policy learning
using \drsdirect always outperforms the baselines, with the
exception of the \emph{optdigits} dataset,
where all the methods perform similarly.\looseness=-1

\subsection{Combinatorial Setting}
\label{sec:exp-slates}

We empirically evaluate the performance of shrinkage-based estimator in the ranking problem
introduced in \pref{ex:slate}. Following~\citet{swaminathan2017off},
we generate contextual bandit data from the fully labeled MSLR-WEB10K dataset~\cite{letor2013}. The dataset has 10K queries, with up to 1251
judged documents for each query. The contexts $x$ are the
queries and actions $\ab$ represent lists of documents. For each
query $x$ and document $i$, the dataset contains a relevance judgement
$\rel(x,i)\in\{0,1,2,3,4\}$. We consider two types of rewards: \emph{deterministic rewards}, $r=\ndcg(x,\ab)$ (see definition in \pref{ex:slate}); and \emph{stochastic rewards}, where $r$ is drawn from a Bernoulli distribution with $p=0.25+0.5\cdot\ndcg(x,\ab)$.
We use data for $10\%$ of the queries to train relevance predictors used to define logging and target policies; the remaining data is used for the bandit protocol. The ground truth is determined using all the data.\looseness=-1

\textbf{Policies.} Each query-document pair $(x,i)$ is described by a feature vector $\fb(x,i)$,
partitioned into title and body features, denoted $\fb_t$ and $\fb_b$. We train two regression models to predict relevance: a lasso model $\lasso_b$ based on $\fb_b$, and a tree model $\tree_t$ based on $\fb_t$. The model $\lasso_b$ is used to select
the top 20 scoring documents; the action $\ab$ is a list of 5 documents out of these 20. In the notation of \pref{ex:slate}, $m=20$, $\ell=5$.
The target policy is deterministic and chooses $\ab$ that lists top 5 documents according to $\tree_t$.
The logging policy is supported on a basis $\Bcal_x\subseteq\Acal$ for each $x$. The basis contains the ``greedy action'' that lists top 5 documents according to $\lasso_b$ as well as actions
obtained by replacing items on the top position and up to two additional positions of the greedy action, resulting in the total of 96 elements in $\Bcal_x$ (see Appendix~\ref{app:exp_slates} for details).
The logging policy is $\epsilon$-greedy:
on each context, $\epsilon$ is drawn uniformly from the set
$\set{2^{-1},2^{-2},2^{-3},2^{-4},2^{-5}}$ and is included as part of the context, creating a
skew in the importance weights $\wb(x)^\top\ab$.\looseness=-1


\textbf{Reward predictors.} We consider two reward predictors $\heta$ trained on logged data. Both are trained via ridge regression, but differ in feature sets they consider: \textit{ridge(all)} is trained on all features, \textit{ridge(5)} is trained on the five features that are most correlated with the reward.

\textbf{Baselines.} We compare our method (DRs-PI) with DM and DR-PI.\footnote{DR-PI dominates self-normalized version of DR-PI as well as the standard pseudo-inverse estimator (i.e., with $\heta\equiv 0$).} In DRs-PI we select the hyperparameter $\lambda$ from
a geometrically spaced grid using our model selection procedure with the empirical version of Eq.~\eqref{eq:bias} in place of bias bound and also consider the oracle tuning of $\lambda$ from the same grid (details in Appendix~\ref{app:exp_slates}).

\textbf{Results and discussion.} In~\pref{fig:slate} we show the MSE of all the methods as a function of sample size, averaged over 20 replicates. Across all conditions, DRs-PI outperforms DR by a factor of 1.5 or more
(note that MSE is reported on log scale). A more striking result is the superior quality of the oracle-tuned DRs-PI.
It shows that the shrinkage strategy is highly effective in achieving a good bias--variance trade-off, but to unlock its potential in combinatorial settings requires improvements in model selection.

\section{Conclusion}
In this paper, we have derived shrinkage-based doubly-robust estimators for
off-policy evaluation using a principled optimization-based
framework. Our approach recovers the weight-clipping estimator from
prior work and also yields novel optimistic shrinkage estimators for
both atomic and combinatorial settings. Extensive experiments
demonstrate the efficacy of these estimators and highlight the role of
model selection in achieving good performance. Thus, the next step is
to develop model selection procedures for off-policy evaluation that
can close the gap with oracle tuning. We look forward to pursuing this
direction in future work.

\section*{Acknowledgements}
This work was partially completed during Yi's and Maria's internships at Microsoft Research. Yi is also supported by the Bloomberg Data Science Fellowship. All content represents the opinion of the authors, which is not necessarily shared or endorsed by their respective employers or sponsors.

\bibliography{refs}

\begin{thebibliography}{31}
\providecommand{\natexlab}[1]{#1}
\providecommand{\url}[1]{\texttt{#1}}
\expandafter\ifx\csname urlstyle\endcsname\relax
  \providecommand{\doi}[1]{doi: #1}\else
  \providecommand{\doi}{doi: \begingroup \urlstyle{rm}\Url}\fi

\bibitem[Awerbuch \& Kleinberg(2008)Awerbuch and Kleinberg]{AwerbuchKl08}
Awerbuch, B. and Kleinberg, R.
\newblock Online linear optimization and adaptive routing.
\newblock \emph{J. Comput. Syst. Sci.}, 74\penalty0 (1):\penalty0 97--114,
  2008.

\bibitem[Bang \& Robins(2005)Bang and Robins]{bang2005doubly}
Bang, H. and Robins, J.~M.
\newblock Doubly robust estimation in missing data and causal inference models.
\newblock \emph{Biometrics}, 2005.

\bibitem[Bembom \& van~der Laan(2008)Bembom and van~der Laan]{bembom2008data}
Bembom, O. and van~der Laan, M.~J.
\newblock Data-adaptive selection of the truncation level for
  inverse-probability-of-treatment-weighted estimators.
\newblock Technical report, UC Berkeley, 2008.

\bibitem[Bottou et~al.(2013)Bottou, Peters, Qui{\~n}onero-Candela, Charles,
  Chickering, Portugaly, Ray, Simard, and Snelson]{bottou2013counterfactual}
Bottou, L., Peters, J., Qui{\~n}onero-Candela, J., Charles, D.~X., Chickering,
  D.~M., Portugaly, E., Ray, D., Simard, P., and Snelson, E.
\newblock Counterfactual reasoning and learning systems: The example of
  computational advertising.
\newblock \emph{The Journal of Machine Learning Research}, 2013.

\bibitem[Cesa-Bianchi \& Lugosi(2012)Cesa-Bianchi and
  Lugosi]{cesa2012combinatorial}
Cesa-Bianchi, N. and Lugosi, G.
\newblock Combinatorial bandits.
\newblock \emph{Journal of Computer and System Sciences}, 2012.

\bibitem[Dani et~al.(2008)Dani, Hayes, and Kakade]{DaniHaKa08}
Dani, V., Hayes, T.~P., and Kakade, S.~M.
\newblock The price of bandit information for online optimization.
\newblock In \emph{Advances in Neural Information Processing Systems}, 2008.

\bibitem[De~la Pena \& Gin{\'e}(2012)De~la Pena and Gin{\'e}]{de2012decoupling}
De~la Pena, V. and Gin{\'e}, E.
\newblock \emph{Decoupling: from dependence to independence}.
\newblock Springer Science \& Business Media, 2012.

\bibitem[Dua \& Graff(2017)Dua and Graff]{Dua:2019}
Dua, D. and Graff, C.
\newblock {UCI} machine learning repository, 2017.
\newblock URL \url{http://archive.ics.uci.edu/ml}.

\bibitem[Dud{\'\i}k et~al.(2011)Dud{\'\i}k, Langford, and Li]{dudik2011doubly}
Dud{\'\i}k, M., Langford, J., and Li, L.
\newblock Doubly robust policy evaluation and learning.
\newblock In \emph{International Conference on Machine Learning}, 2011.

\bibitem[Dud{\'\i}k et~al.(2014)Dud{\'\i}k, Erhan, Langford, Li,
  et~al.]{dudik2014doubly}
Dud{\'\i}k, M., Erhan, D., Langford, J., Li, L., et~al.
\newblock Doubly robust policy evaluation and optimization.
\newblock \emph{Statistical Science}, 2014.

\bibitem[Farajtabar et~al.(2018)Farajtabar, Chow, and
  Ghavamzadeh]{farajtabar2018more}
Farajtabar, M., Chow, Y., and Ghavamzadeh, M.
\newblock More robust doubly robust off-policy evaluation.
\newblock In \emph{International Conference on Machine Learning}, 2018.

\bibitem[Hahn(1998)]{hahn1998role}
Hahn, J.
\newblock On the role of the propensity score in efficient semiparametric
  estimation of average treatment effects.
\newblock \emph{Econometrica}, 1998.

\bibitem[Hirano et~al.(2003)Hirano, Imbens, and Ridder]{hirano2003efficient}
Hirano, K., Imbens, G.~W., and Ridder, G.
\newblock Efficient estimation of average treatment effects using the estimated
  propensity score.
\newblock \emph{Econometrica}, 2003.

\bibitem[Imbens et~al.(2007)Imbens, Newey, and Ridder]{imbens2007mean}
Imbens, G., Newey, W., and Ridder, G.
\newblock Mean-squared-error calculations for average treatment effects.
\newblock \emph{ssrn.954748}, 2007.

\bibitem[Kallus(2017)]{Kallus17}
Kallus, N.
\newblock {A Framework for Optimal Matching for Causal Inference}.
\newblock In \emph{International Conference on Artificial Intelligence and
  Statistics}, 2017.

\bibitem[Kallus(2018)]{Kallus18}
Kallus, N.
\newblock Balanced policy evaluation and learning.
\newblock In \emph{Advances in Neural Information Processing Systems}, 2018.

\bibitem[Kang et~al.(2007)Kang, Schafer, et~al.]{kang2007demystifying}
Kang, J.~D., Schafer, J.~L., et~al.
\newblock Demystifying double robustness: A comparison of alternative
  strategies for estimating a population mean from incomplete data.
\newblock \emph{Statistical science}, 2007.

\bibitem[Langford \& Zhang(2008)Langford and Zhang]{langford2008epoch}
Langford, J. and Zhang, T.
\newblock The epoch-greedy algorithm for multi-armed bandits with side
  information.
\newblock In \emph{Advances in Neural Information Processing Systems}, 2008.

\bibitem[Li et~al.(2011)Li, Chu, Langford, and Wang]{li2011unbiased}
Li, L., Chu, W., Langford, J., and Wang, X.
\newblock Unbiased offline evaluation of contextual-bandit-based news article
  recommendation algorithms.
\newblock In \emph{International Conference on Web Search and Data Mining},
  2011.

\bibitem[Li et~al.(2015)Li, Chen, Kleban, and Gupta]{li2015counterfactual}
Li, L., Chen, S., Kleban, J., and Gupta, A.
\newblock Counterfactual estimation and optimization of click metrics in search
  engines: A case study.
\newblock In \emph{International Conference on World Wide Web}, 2015.

\bibitem[Qin \& Liu(2013)Qin and Liu]{letor2013}
Qin, T. and Liu, T.
\newblock Introducing {LETOR} 4.0 datasets.
\newblock \emph{arXiv:1306.2597}, 2013.

\bibitem[Robins \& Rotnitzky(1995)Robins and
  Rotnitzky]{robins1995semiparametric}
Robins, J.~M. and Rotnitzky, A.
\newblock Semiparametric efficiency in multivariate regression models with
  missing data.
\newblock \emph{Journal of the American Statistical Association}, 1995.

\bibitem[Rothe(2016)]{rothe2016value}
Rothe, C.
\newblock The value of knowing the propensity score for estimating average
  treatment effects.
\newblock \emph{IZA Discussion Paper Series}, 2016.

\bibitem[Strehl et~al.(2010)Strehl, Langford, Li, and
  Kakade]{strehl2010learning}
Strehl, A., Langford, J., Li, L., and Kakade, S.~M.
\newblock Learning from logged implicit exploration data.
\newblock In \emph{Advances in Neural Information Processing Systems}, 2010.

\bibitem[Su et~al.(2018)Su, Wang, Santacatterina, and Joachims]{su2018cab}
Su, Y., Wang, L., Santacatterina, M., and Joachims, T.
\newblock Cab: Continuous adaptive blending estimator for policy evaluation and
  learning.
\newblock In \emph{International Conference on Machine Learning}, 2018.

\bibitem[Swaminathan \& Joachims(2015{\natexlab{a}})Swaminathan and
  Joachims]{swaminathan2015counterfactual}
Swaminathan, A. and Joachims, T.
\newblock Counterfactual risk minimization: Learning from logged bandit
  feedback.
\newblock In \emph{International Conference on Machine Learning},
  2015{\natexlab{a}}.

\bibitem[Swaminathan \& Joachims(2015{\natexlab{b}})Swaminathan and
  Joachims]{swaminathan2015self}
Swaminathan, A. and Joachims, T.
\newblock The self-normalized estimator for counterfactual learning.
\newblock In \emph{Advances in Neural Information Processing Systems},
  2015{\natexlab{b}}.

\bibitem[Swaminathan et~al.(2017)Swaminathan, Krishnamurthy, Agarwal, Dudik,
  Langford, Jose, and Zitouni]{swaminathan2017off}
Swaminathan, A., Krishnamurthy, A., Agarwal, A., Dudik, M., Langford, J., Jose,
  D., and Zitouni, I.
\newblock Off-policy evaluation for slate recommendation.
\newblock In \emph{Advances in Neural Information Processing Systems}, 2017.

\bibitem[Thomas \& Brunskill(2016)Thomas and Brunskill]{thomas2016data}
Thomas, P. and Brunskill, E.
\newblock Data-efficient off-policy policy evaluation for reinforcement
  learning.
\newblock In \emph{International Conference on Machine Learning}, 2016.

\bibitem[Wang et~al.(2017)Wang, Agarwal, and Dudik]{wang2017optimal}
Wang, Y.-X., Agarwal, A., and Dudik, M.
\newblock Optimal and adaptive off-policy evaluation in contextual bandits.
\newblock In \emph{International Conference on Machine Learning}, 2017.

\bibitem[Zhou et~al.(2017)Zhou, Mayer-Hamblett, Khan, and
  Kosorok]{zhou2017residual}
Zhou, X., Mayer-Hamblett, N., Khan, U., and Kosorok, M.~R.
\newblock Residual weighted learning for estimating individualized treatment
  rules.
\newblock \emph{Journal of the American Statistical Association}, 2017.

\end{thebibliography}
\bibliographystyle{icml2020}

\vfill
\newpage

\appendix
\onecolumn
\section{Derivation of Shrinkage Estimators for Non-combinatorial Setting}
\label{app:examples}
In this section we provide detailed derivations for the two
estimators in the non-combinatorial setting.


We first derive the pessimistic version. Recall that the optimization
problem decouples across $(x,a)$, so we focus on a single $(x,a)$
pair such that $\mu(a\given x)>0$ since only such pairs can appear
in the data. For conciseness, we omit the dependence on $(x,a)$ and simply write $w=w(x,a)$, $\hw=\hw(x,a)$ and
$\mu=\mu(a\given x)$. Fixing $\lambda \geq
0$, we must solve
\begin{align*}
\Minimize_{\hw\in\RR}
\;
\BigBracks{\mu\hw^2 + 2\lambda \abs{\mu (\hw - w)}}. 
\end{align*}
(Note that we allow any $\hw\in\RR$, but we will see that the solution will actually satisfy $0\le\hw\le w$.)
Since $\mu>0$, this minimization problem is strongly convex and therefore has a unique minimizer.
By first-order optimality, $\hw$ is a minimizer if and only if
\begin{align}
\label{eq:pes-optimality}
2 \mu \hw + 2\lambda\mu v = 0
\qquad\text{and}\qquad
v \in \partial \abs{\hw - w} = \left\{
\begin{aligned}
1 & \textrm{ if } & \hw > w,\\
[-1, 1] & \textrm{ if } & \hw = w,\\
-1 & \textrm{ if } & \hw < w.
\end{aligned}\right.
\end{align}
Since $\mu>0$, the first equation can be rewritten as
\[
   \hw = -\lambda v.
\]
Now a simple case analysis shows that if $w>\lambda$ then the choice $\hw=\lambda$, $v=-1$ satisfies
Eq.~\eqref{eq:pes-optimality}, and if $0\le w\le\lambda$ then the choice $\hw=w$, $v=-w/\lambda$ satisfies Eq.~\eqref{eq:pes-optimality},
yielding
\begin{align*}
\hw_{\pes,\lambda}(x,a) = \min\{\lambda, w(x,a)\},
\end{align*}
which is the clipped estimator.

For the optimistic version, the optimization problem is
\begin{align*}
\Minimize_{\hw\in\RR}
\;
\BigBracks{\mu \hw^2 w^2/z + \lambda \mu (\hw - w)^2/z},
\end{align*}
where $z=z(x,a)$.
The optimality
conditions are
\begin{align*}
2 \mu w^2\hw/z + 2 \lambda \mu(\hw - w)/z = 0.
\end{align*}
This gives the optimistic estimator
\begin{align*}
\hw_{\opt,\lambda}(x,a) = \frac{\lambda}{w(x,a)^2 + \lambda} w(x,a).
\end{align*}
Notice that this estimator does not depend on the weighting function $z$, so it
does not depend on how we train the regression model.

\section{Derivation of the Shrinkage Estimator for Combinatorial Setting}
\label{app:slates}

We provide a complete derivation of the shrinkage estimator for combinatorial
actions. We use the notation $w(x,\ab)=\wb(x)^\top\ab$.
We assume that the regression model takes form $\heta(x,\ab)=\hetab(x)^\top\ab$, satisfies $\heta(x,\ab)\in[0,1]$, and is trained to minimize
\begin{align*}
L(\hetab) \defeq \EE_{\mu}\sbr{ z(x,\ab) (r - \hetab(x)^\top\ab)^2}
\end{align*}
for some $z(x,\ab)>0$.
We assume that the linearity
assumption holds, so we can write $\eta(x,\ab)=\etab(x)^\top\ab$. And we also
assume that
$
  \linspan\bigParens{\supp\pi(\cdot\given x)}
  \subseteq
  \linspan\bigParens{\supp\mu(\cdot\given x)}
$, so, as shown by~\citet{swaminathan2017off}, the pseudo-inverse estimator
is unbiased:
\begin{equation}
\label{eq:app-bias}
  \EE_{(x,\ab,r)\sim\mu}\BigBracks{\hetab(x)^\top\qb_{\pi,x}+w(x,\ab)\bigParens{r-\hetab(x)^\top\ab}}
  =\EE_{(x,\ab,r)\sim\pi}[r].
\end{equation}
Therefore, if we replace $w$ by an arbitrary function $\hw$ (not necessarily linear in $\ab$), we
obtain the expression for the bias
\begin{equation}
\bias(\hw) = \EE_{\mu}\sbr{\bigParens{\hw(x,\ab) - w(x,\ab)}(r - \hetab(x)^\top\ab)}.
\end{equation}
Using Cauchy--Schwarz inequality, we can bound the bias in terms of $L(\hetab)$:
\begin{align*}
\bias(\hw)
\leq \sqrt{\EE_\mu\sbr{\bigParens{\hw(x,\ab) - w(x,\ab)}^2/z(x,\ab)}} \cdot \sqrt{L(\hetab)}.
\end{align*}

For the variance bound, we begin with a proxy based on \pref{prop:variance_pi}, and then bound it
using Cauchy--Schwarz inequality, the fact that $\hetab(x)^\top\ab$ and $r$ are bounded in $[0,1]$, and an additional assumption that $\abs{\hw(x,\ab)}\le\abs{w(x,\ab)}$ (which
we will show is true for the specific optimistic estimator that we derive below):
\begin{align*}
\Var(\hw) &\approx
\frac{1}{n}\EE_\mu\sbr{\hw(x,\ab)^2\bigParens{r - \hetab(x)^\top\ab}^2}
\leq \frac{1}{n}\sqrt{\EE_{\mu}\bigBracks{ \hw(x,\ab)^4/z(x,\ab)}}\cdot\sqrt{\EE_\mu\sbr{z(x,\ab)\bigParens{r - \hetab(x)^\top\ab}^4}}\\
& \leq \frac{1}{n}\sqrt{\EE_\mu\BigBracks{ \hw(x,\ab)^2 w(x,\ab)^2/z(x,\ab)}}\cdot \sqrt{L(\hetab)}.
\end{align*}

Similar to non-combinatorial setting, the solutions of the resulting MSE bound must lie on the Pareto front parameterized by a single scalar
$\lambda \in [0,\infty]$:
\begin{align*}
\Minimize_{\hw}\ \lambda \EE_\mu\sbr{\tfrac{1}{z(x,\ab)}\bigParens{\hw(x,\ab) - w(x,\ab)}^2} + \EE_{\mu}\sbr{\tfrac{w(x,\ab)^2}{z(x,\ab)}\hw(x,\ab)^2 }.
\end{align*}
This decomposes across $(x,\ab)$ and by first-order optimality, we obtain the same solution as in non-combinatorial setting:
\[
  \hw(x,\ab) = \frac{\lambda}{w(x,\ab)^2 + \lambda} w(x,\ab).
\]
Note that these weights satisfy $\abs{\hw(x,\ab)}\le\abs{w(x,\ab)}$, and $\hw$ matches the sign of $w$, so
in fact a stronger shrinkage property holds: $\hw(x,\ab)=c(x,\ab)w(x,\ab)$ where $c(x,\ab)\in[0,1]$.  Also
note that when $\mu$ is supported on a linearly independent set of actions for any given
$x$, then we can pick $\hwb(x)\in\RR^d$ to satisfy $\hw(x,\ab)=\hwb(x)^\top\ab$ across all actions in the support of $\mu(\cdot\given x)$, thus
satisfying the assumptions of \pref{prop:variance_pi}.
Plugging the expression for $\hw$ back
into the pseudo-inverse estimator yields
\begin{align*}
\vdrsopi(\pi;\hetab,\lambda)
\defeq\frac{1}{n}\sum_{i=1}^n \hetab(x_i)^\top\qb_{\pi,x_i} + \rbr{\frac{\lambda}{\lambda + (\wb(x_i)^\top\ab_i)^2}}\wb(x_i)^\top\ab_i (r_i - \hetab(x_i)^\top\ab_i).
\end{align*}

\section{Proofs}

\subsection{Proof of~\pref{prop:variance}}

The law of total variance gives
\begin{align*}
\Var(\hat{w}) &= \frac{1}{n}\Var_{x,a,r\sim\mu} \rbr{ \sum_{a' \in \Acal} \pi(a'\given x)\heta(x,a') + \hw(x,a)(r - \heta(x,a))}\\
& = \frac1n\underbrace{ \EE_{x}\Var_{a,r \sim\mu} \rbr{ \sum_{a' \in \Acal} \pi(a'\given x)\heta(x,a') + \hw(x,a)(r - \heta(x,a))}}_{\eqdef T_1}\\
&\quad{} + \frac1n\underbrace{\Var_{x} \EE_{a,r \sim\mu} \rbr{ \sum_{a' \in \Acal} \pi(a'\given x)\heta(x,a') + \hw(x,a)(r - \heta(x,a))}}_{\eqdef T_2}.
\end{align*}
For $T_1$, since $\sum_{a' \in \Acal} \pi(a' \given
x)\heta(x,a')$ does not depend on $a,r$, it does not contribute to the
conditional variance, and we get
\begin{align*}
T_1 & = \EE_x \Var_{a,r} \rbr{\hat{w}(x,a)(r - \heta(x,a))} = \EE_{x,a,r} \sbr{\hw(x,a)^2(r - \heta(x,a))^2 } - \EE_x \sbr{\EE_{a,r}\sbr{\hw(x,a)(r - \heta(x,a))}^2}.
\end{align*}
The first term is our variance proxy. To bound the second term, write $\hw(x,a)=c(x,a)w(x,a)$ for some $c(x,a)\in[0,1]$, which is possible since $0\le w\le\hw$ by assumption. The second term
can then be rewritten and bounded as
\begin{align*}
0\le\EE_x \sbr{\EE_{a,r\sim\mu}\sbr{\hw(x,a)\bigParens{r - \heta(x,a)}}^2}
  &= \EE_x \sbr{\EE_{a,r\sim\mu}\sbr{w(x,a)c(x,a)\bigParens{r - \heta(x,a)}}^2}
\\
  &= \EE_x \sbr{\EE_{a,r\sim\pi}\sbr{c(x,a)\bigParens{r - \heta(x,a)}}^2}
  \leq 1,
\end{align*}
where the second equality follows by the unbiasedness of inverse-propensity scoring,
and the final bound follows because $c(x,a), r, \heta(x,a)\in[0,1]$.


For $T_2$, of course we have $T_2 \geq 0$, and further
\begin{align*}
T_2
&=
\Var_{x} \Bracks{\EE_{a\sim\pi}\bigBracks{\heta(x,a)}
                +\EE_{a\sim\mu}\bigBracks{\hw(x,a)\bigParens{\eta(x,a)-\heta(x,a)}}}
\\
&
\le
\E_{x} \Bracks{\Parens{\EE_{a\sim\pi}[\heta(x,a)]
                +\EE_{a\sim\mu}\bigBracks{\hw(x,a)\bigParens{\eta(x,a)-\heta(x,a)}}}^2}
\\
&
=
\E_{x} \Bracks{\Parens{\EE_{a\sim\pi}[\heta(x,a)]
                +\EE_{a\sim\mu}\bigBracks{w(x,a)c(x,a)\bigParens{\eta(x,a)-\heta(x,a)}}}^2}
\\
&
=
\E_{x} \Bracks{\EE_{a\sim\pi}\BigBracks{\heta(x,a)+ c(x,a)\bigParens{\eta(x,a)-\heta(x,a)}}^2}
\\
&
=
\E_{x} \Bracks{\EE_{a\sim\pi}\BigBracks{\bigParens{1-c(x,a)}\heta(x,a)+ c(x,a)\eta(x,a)}^2}
\le 1,
\end{align*}
where we again write $\hw(x,a)=c(x,a)w(x,a)$ for some $c(x,a)\in[0,1]$, then appeal to the
unbiasedness of the inverse-propensity scoring, and finally use the bounds
$c(x,a),\eta(x,a),\hat{\eta}(x,a) \in [0,1]$.

Therefore, we can write
\begin{align*}
\Var(\hw) - \frac1n\EE_{x,a,r} \sbr{\hw(x,a)^2(r - \heta(x,a))^2 } = -\frac1n\EE_x\sbr{\EE_{a,r}\sbr{\hw(x,a)(r - \heta(x,a))}^2} + \frac1n T_2,
\end{align*}
and we have just shown that the right hand side is in $\bigBracks{-\frac1n,\frac1n}$. This
proves the proposition.


\subsection{Proof of \pref{prop:variance_pi}}

We begin by deriving a simple expression for the
pseudo-inverse $\Gamma_{\mu,x}^\dagger$. Consider a fixed $x$
and let $s=\card{\Bcal_x}$ be the size of the basis $\Bcal_x$ (note that
$s$ might be a function of $x$). Let $D_{\mu,x}\in\RR^{s\times s}$ denote
the diagonal matrix $D_{\mu,x}=\diag\set{\mu(\ab\given x)}_{\ab\in\Bcal_x}$.
Recall that $B_x$ is the matrix with $\ab\in\Bcal_x$ in its columns.
The matrix $\Gamma_{\mu,x}$ can then be written as
\[
  \Gamma_{\mu,x}=B_x D_{\mu,x} B_x^\top.
\]
To obtain its psedo-inverse, we use tho following fact:
\begin{fact}
Let $B\in\RR^{d\times s}$ be a matrix with linearly independent columns and
let $K=B^\top B$. Then for any invertible diagonal matrix $D \in \RR^{s \times s}$, we have
\begin{align*}
(BDB^\top)^\dagger = BK^{-1}D^{-1}K^{-1}B^\top,
\end{align*}
where $K^{-1}$ is well defined thanks to the linear independence of columns of $B$.
\end{fact}
\begin{proof}
Let $G \defeq B DB^\top$ and $G'\defeq BK^{-1}D^{-1}K^{-1}B^\top$. To show that $G^\dagger=G'$, it
suffices to argue that $GG'G=G$ and $G'GG'=G'$:
\begin{align*}
GG'G &= B DB^\top B K^{-1} D^{-1} K^{-1} B^\top B D B^\top = B D B^\top = G,\\
G'G G' &= BK^{-1}D^{-1} K^{-1}B^\top BD B^\top BK^{-1}D^{-1} K^{-1}B^\top = BK^{-1}D^{-1} K^{-1}B^\top = G'.\tag*\qedhere
\end{align*}
\end{proof}
Using this fact, we thus have
\begin{equation}
\label{eq:gamma-dagger}
  \Gamma_{\mu,x}^\dagger=B_x K_x^{-1}D_{\mu,x}^{-1}K_x^{-1}B_x^\top,
\end{equation}
where $K_x=B_x^\top B_x$.

We are now ready to start the proof of \pref{prop:variance_pi}.
Similarly to the proof of~\pref{prop:variance}, we first apply the law of total variance
\begin{align*}
n \Var(\hw)
&= \Var_{x,\ab,r\sim\mu}\rbr{ \hetab(x)^\top\qb_{\pi,x} + \hwb(x)^\top\ab(r - \hetab(x)^\top\ab)}
\\
& = \underbrace{\EE_{x} \Var_{\ab,r \sim \mu}\rbr{ \hetab(x)^\top\qb_{\pi,x} + \hwb(x)^\top\ab(r - \hetab(x)^\top\ab)}}_{=: T_1} + \underbrace{\Var_x \EE_{\ab,r\sim\mu}\rbr{ \hetab(x)^\top\qb_{\pi,x} + \hwb(x)^\top\ab(r - \hetab(x)^\top\ab)}}_{=: T_2}.
\end{align*}
To analyze $T_2$, we first rewrite and bound the inner expectation for a fixed $x$. We drop the dependence on
$x$ from the notation and write $\hetab=\hetab(x)$, $\qb_\pi=\qb_{\pi,x}$, $\hwb=\hwb(x)$, $\etab=\etab(x)$, and $\Gamma_\mu=\Gamma_{\mu,x}$:
\begin{align}
\notag
&\EE_{\ab,r\sim\mu}\BigBracks{ \hetab^\top\qb_{\pi} + \hwb^\top\ab(r - \ab^\top\hetab)}
\\
\notag
&\qquad\qquad{}
=
\hetab^\top\qb_\pi
+
\EE_{\ab\sim\mu}\BigBracks{\hwb^\top\ab(\ab^\top\etab - \ab^\top\hetab)}
=
\hetab^\top\qb_\pi
+
\hwb^\top\Parens{\EE_{\ab\sim\mu}[\ab\ab^\top]}(\etab-\hetab)
=
\hetab^\top\qb_\pi
+
\hwb^\top\Gamma_{\mu}(\etab-\hetab)
\\
\label{eq:app-slates-1}
&\qquad\qquad{}
=
\hetab^\top B \vb_\pi
+
\hwb^\top(B D_\mu B^\top)(\etab-\hetab)
\end{align}
where in the last step we introduced shorthands $B=B_x$, $D_\mu=D_{\mu,x}$ and $\vb_\pi=\vb_{\pi,x}$.

To continue with the derivation, observe that
by assumption, we have $\hat{\wb}^\top\ab = c(x,\ab) \wb^\top\ab$ for all $\ab\in\Bcal_x$,
and so we can write $\hat{\wb}^\top B = \wb^\top B C$ where $C$ is a diagonal
matrix with entries $c(x,\ab)$ across $\ab\in\Bcal_x$. Next, using the fact that
$\wb=\Gamma_\mu^\dagger\qb_\pi$ and then plugging in Eq.~\eqref{eq:gamma-dagger}, we obtain
\begin{equation}
\label{eq:app-slates-2}
\hat{\wb}^\top B
= \wb^\top B C
= \qb_\pi^\top\Gamma_\mu^\dagger B C
= \vb_\pi^\top B^\top\BigParens{B K^{-1}D_{\mu}^{-1}K^{-1}B^\top} B C
= \vb_\pi^\top D_{\mu}^{-1} C,
\end{equation}
where we introduced the shorthand $K=K_x$. Now combining Eqs.~\eqref{eq:app-slates-1} and~\eqref{eq:app-slates-2}, we obtain
\begin{align}
\notag
\Abs{\EE_{\ab,r\sim\mu} \BigBracks{\hetab^\top\qb_\pi + \hat{\wb}^\top\ab(r - \ab^\top\hetab)}}
&
= \BigAbs{\vb_\pi^\top B^\top \hetab + \vb_\pi^\top D_{\mu}^{-1} C D_\mu B^\top(\etab - \hetab)}
\\
\notag
& = \Abs{\vb_\pi^\top\BigParens{B^\top \hetab + CB^\top(\etab -\hetab)}}
\\
\label{eq:app-slates-3}
& \leq \norm{\vb_\pi}_{1} \cdot \Norm{(I - C)B^\top \hetab + C B^\top\etab}_{\infty},
\end{align}
where in the last step we used Holder's inequality.
For the $\ell_{\infty}$ norm, we get
\begin{align*}
\Norm{(I - C)B^\top \hetab + C B^\top\etab}_{\infty}^2 &=
\max_{\ab\in\Bcal_x}
\BigAbs{\bigParens{1-c(x,\ab)}\heta(x,\ab) + c(x,\ab)\eta(x,\ab)}
\le 1,
\end{align*}
where the last step follows because $\eta(x,\ab),\heta(x,\ab),c(x,\ab)\in[0,1]$.
Therefore, we have that $0 \leq T_2 \leq \EE_x\bigBracks{\norm{\vb_{\pi,x}}_1^2}$.

To bound $T_1$, we first note that $\hetab(x)^\top\qb_{\pi,x}$ is independent of $\ab$ and $r$, and so it does not contribute to the variance, and so
\begin{align*}
T_1
= \EE_{x,\ab,r\sim\mu}\BigBracks{\rbr{\hat{\wb}^\top\ab(r - \ab^\top\hetab)}^2}
- \EE_{x}\BigBracks{\,\EE_{\ab,r\sim\mu}\sbr{\hat{\wb}^\top\ab(r - \ab^\top\hetab)}^2}.
\end{align*}
The second term satisfies
\[
0
\leq \EE_{x}\BigBracks{\,\EE_{\ab,r\sim\mu}\sbr{\hat{\wb}^\top\ab(r - \ab^\top\hetab)}^2}
\leq \EE_x \BigBracks{\norm{\vb_{\pi,x}}_1^2\cdot\bigNorm{CB^\top(\etab - \hetab)}_{\infty}^2}
\leq \EE_x \bigBracks{\norm{\vb_{\pi,x}}_1^2},
\]
where we applied similar reasoning as in Eq.~\eqref{eq:app-slates-3}. Combining
this bound with the bound on $T_2$ completes the proof:
\[
\BigAbs{
 n \Var(\hw)-
 \EE_{x,\ab,r\sim\mu}\BigBracks{\rbr{\hat{\wb}^\top\ab(r - \ab^\top\hetab)}^2}
}
\le
\EE_x \bigBracks{\norm{\vb_{\pi,x}}_1^2}.
\]

\subsection{Proof of~\pref{thm:model_selection}}

The main technical part of the proof is a deviation inequality for the
sample variance. For this, let us fix $\theta$, which we drop from
notation, and focus on estimating the variance
\begin{align*}
\Var(Z) = \EE[(Z - \EE(z))^2] \textrm{ with } \hvar = \frac{1}{2n(n-1)}\sum_{\substack{1\le i,j\le n\\i\ne j}}(Z_i - Z_j)^2.
\end{align*}
We have the following lemma
\begin{lemma}[Variance estimation]
\label{lem:variance_estimation}
Let $Z_1,\ldots,Z_n$ be iid random variables, and assume that $|Z_i|
\leq R$ almost surely. Then there exists a constant $C > 0$ such that
for any $\delta \in (0,1)$, with probability at least $1-\delta$
\begin{align*}
\abr{\Var(Z) - \hvar} \leq (C+3) \rbr{\sqrt{\frac{2 R^2 \Var(Z)\log(6C/\delta)}{n}} + \frac{2R^2\log(6C/\delta)}{3n}}.
\end{align*}
\end{lemma}
\begin{proof}
For this lemma only, define $\mu = \EE[Z]$. By direct calculation
\begin{align*}
\Var(Z) = \EE\sbr{Z^2} - \mu^2, \qquad \hvar = \frac{1}{n}\sum_{i=1}^n Z_i^2 - \frac{1}{n(n-1)}\sum_{i \ne j} Z_iZ_j.
\end{align*}
We work with the second term first. Let $Z_1',\ldots,Z_n'$ be an iid
sample, independent of $Z_1,\ldots,Z_n$. Now, by Theorem 3.4.1 of~\citet{de2012decoupling}, we have
\begin{align*}
\PP\sbr{\abr{\frac{1}{n(n-1)} \sum_{i \ne j} (Z_i - \mu)(Z_j-\mu)} > t} \leq C \PP\sbr{\abr{\frac{1}{n(n-1)} \sum_{i \ne j} (Z_i - \mu)(Z_j'-\mu)} > t/C}
\end{align*}
for a universal constant $C>0$. Thus, we have decoupled the
U-statistic. Now let us condition on $Z_1,\ldots,Z_n$ and write $X_j =
\frac{1}{n-1}\sum_{i \ne j}(Z_i - \mu)$, which conditional on
$Z_1,\ldots,Z_n$ is non-random. We will apply Bernstein's inequality
on $\frac{1}{n}\sum_{j=1}^nX_j(Z'_j-\mu)$, which is a centered random
variable, conditional on $Z_{1:n}$. This gives that with probability
at least $1-\delta$
\begin{align*}
\abr{\frac{1}{n}\sum_{j=1}^{n}X_j (Z'_j-\mu)} &\leq \sqrt{\frac{2\frac{1}{n}\sum_{j=1}^n\Var(X_jZ'_j)\log(2/\delta)}{n}} + \frac{2 \max_j \sup \abr{X_j(Z_j-\mu)} \log(2/\delta)}{3n}.\\
& \leq \max_j|X_j| \rbr{\sqrt{\frac{2\Var(Z)\log(2/\delta)}{n}} + \frac{2 R \log(2/\delta)}{3n}}.
\end{align*}
This bound holds with high probability for any $\{X_j\}_{j=1}^n$. In
particular, since $|X_j| \leq R$ almost surely, we get that with
probability $1-\delta$
\begin{align*}
\abr{\frac{1}{n(n-1)} \sum_{i \ne j} (Z_i - \mu)(Z_j - \mu)} \leq C \sqrt{\frac{2R^2\Var(Z)\log(2C/\delta)}{n}} + \frac{2CR^2\log(2C/\delta)}{3n}.
\end{align*}
The factors of $C$ arise from working through the decoupling
inequality.

Next, by a standard application of Bernstein's inequality, with probability at least $1-\delta$, we have
\begin{align*}
\abr{\frac{1}{n}\sum_{i=1}^n Z_i - \mu} \leq \sqrt{\frac{2 \Var(Z)\log(2/\delta)}{n}} + \frac{2R\log(2/\delta)}{3n}.
\end{align*}
Therefore, with probability $1-2\delta$ we have
\begin{align*}
\abr{\frac{1}{n(n-1)}\sum_{i \ne j} Z_i Z_j - \mu^2} &\leq \abr{\frac{1}{n(n-1)}\sum_{i\ne j} (Z_i - \mu)(Z_j-\mu)} + 2\abr{\frac{1}{n}\sum_{i=1}^n Z_i\mu - \mu^2}\\
& \leq \abr{\frac{1}{n(n-1)}\sum_{i\ne j} (Z_i - \mu)(Z_j-\mu)} + 2R \abr{\frac{1}{n}\sum_{i=1}^n Z_i - \mu}\\
& \leq (C+2)\sqrt{\frac{R^2\Var(Z)\log(2C/\delta)}{n}} + \frac{2(C+2)R^2\log(2C/\delta)}{3n}
\end{align*}


Let us now address the first term, a simple application of Bernstein's
inequality gives that with probability at least $1-\delta$
\begin{align*}
\abr{\frac{1}{n}\sum_{i=1}^{n}Z_i^2 - \EE[Z^2]} \leq \sqrt{\frac{2 \Var(Z^2)\log(2/\delta)}{n}} + \frac{2 R^2\log(2/\delta)}{3n}\\
\leq \sqrt{\frac{2 R^2 \Var(Z)\log(2/\delta)}{n}} + \frac{2 R^2\log(2/\delta)}{3n}.
\end{align*}
Combining the two inequalities, we obtain the result.
\end{proof}

Since we are estimating the variance of the sample average estimator,
we divide by another factor of $n$. Meanwhile the range and the
variance terms themselves are certainly $O(1)$, so the error terms in
Lemma~\ref{lem:variance_estimation} are $O(n^{-3/2})$ and $O(n^{-2})$
respectively. Formally, there exists a universal constants $C_1,C_2>0$ such
that for any $\delta \in (0,1)$ with probability at least $1-\delta$
we have
\begin{align*}
\abr{ \Var(\theta) - \hvar(\theta)} \leq C_1\sqrt{\frac{\log(1/\delta)}{n^3}} + C_2\frac{\log(1/\delta)}{n^2}.
\end{align*}
By adjusting the constant, we can simplify the expression by removing
the $n^{-2}$ term. In other words, there exists a different universal
constant $C > 0$ such that
\begin{align*}
\abr{ \Var(\theta) - \hvar(\theta)} \leq \frac{C\log(1/\delta)}{n^{3/2}}
\end{align*}
holds with probability at least $1-\delta$.


For the model selection result, first apply
Lemma~\ref{lem:variance_estimation} for all $\theta \in \Theta$,
taking a union bound. Further take a union bound over the event that
$\bias(\theta) \leq \hbias(\theta)$ for all $\theta \in \Theta$, if it
is needed. Then, observe that for any $\theta_0 \in \Theta_0$ we have
\begin{align*}
\mse(\hat{\theta}) & = \bias(\hat{\theta})^2 + \Var(\hat{\theta})
\leq \hbias(\hat{\theta})^2 + \hvar(\hat{\theta}) + \frac{C\log(|\Theta|/\delta)}{n^{3/2}}\\
& \leq \hbias(\theta_0)^2 + \hvar(\theta_0) + \frac{C\log(|\Theta|/\delta)}{n^{3/2}}
 \leq 0 + \Var(\theta_0) + \frac{2C\log(|\Theta|/\delta)}{n^{3/2}} = \mse(\theta_0) + \frac{2C\log(|\Theta|/\delta)}{n^{3/2}}.
\end{align*}
The first inequality uses Lemma~\ref{lem:variance_estimation} and the
fact that $\bias \leq \hbias$. The second uses that $\hat{\theta}$
optimizes this quantity, and the third uses the property that
$\hbias(\theta_0) = 0$ by assumption. Note that the universal constant
here is slightly different from the one in the variance bound, since
we have also taken a union bound for the bias term.

\subsection{Construction of Upper Bounds on Bias}

In this section we give detailed construction of bias upper bounds
that we use in the model selection procedure. Recall that this is for
the analysis only. Empirically we found that using the estimators
alone --- not the upper bounds --- leads to better performance.

Throughout, we fix a set of hyperparameters $\theta$, which we
suppress from the notation.

\paragraph{Direct bias estimation.}
The most straightforward bias estimator is to simply approximate the
expectation with a sample average.
\begin{align*}
\tbias = \Abs{\frac{1}{n}\sum_{i=1}^{n} \rbr{\hat{w}(x_i,a_i) - w(x_i,a_i)}\rbr{r_i - \heta(x_i,a_i)}}.
\end{align*}
This estimator has finite-sum structure, and naively, each term is
bounded in $[-w_\infty,w_\infty]$ where $w_\infty = \max_{x,a} w(x,a)$. The variance is at most
$\EE_{\mu} [w(x,a)^2]$. Hence Bernstein's inequality gives that with probability at least $1-\delta$
\begin{align*}
\abr{\tbias - \bias} \leq \sqrt{\frac{2 \EE_\mu [w(x,a)^2] \log(2/\delta)}{n}} + \frac{2w_{\infty} \log(2/\delta)}{3n}.
\end{align*}
Inflating the estimate by the right hand side gives $\hbias$, which is
a high probability upper bound on $\bias$.

\paragraph{Pessimistic estimation.}
The bias bound used in the pessimistic estimator and its natural
sample estimator are
\begin{align*}
\EE_\mu\bigBracks{\bigAbs{\hat{w}(x,a) - w(x,a)}}, \qquad
\tbias
= \frac{1}{n}\sum_{i} \sum_a \mu(a\given x_i) \BigAbs{\hat{w}(x_i,a) - w(x_i,a)}
= \frac{1}{n}\sum_{i} \sum_a \pi(a\given x_i) \Abs{\frac{\hat{w}(x_i,a)}{w(x_i,a)}-1}.
\end{align*}
Note that since we have already eliminated the dependence on the
reward, we can analytically evaluate the expectation over actions,
which will lead to lower variance in the estimate.

Again we perform a fairly naive analysis. Since $0\le\hat{w}(x,a) \leq
w(x,a)$, the random variables, equal to the inner sum over~$a$, take values in $[0,1]$. Therefore, Hoeffding's inequality gives that with probability $1-\delta$
\begin{align*}
\bias \leq \tbias + \sqrt{\frac{\log(1/\delta)}{2n}},
\end{align*}
and we use the right hand side for our high probability upper bound.

\paragraph{Optimistic estimation.}
For the optimistic bound, we must estimate two terms, one involving
the regressor and one involving the importance weights. We use
\begin{align*}
T_1 \defeq \frac{1}{n}\sum_{i=1}^{n} z(x_i,a_i)(r_i - \heta(x_i,a_i))^2, \qquad
T_2 \defeq
    \frac{1}{n}\sum_{i=1}^{n} \sum_a
           \mu(a\given x_i)\frac{\abs{\hat{w}(x_i,a) - w(x_i,a)}^2}{z(x_i,a)}.
\end{align*}
Note here that the former uses sampled actions from $\mu$, but does
not involve the importance weight, while the latter involves the
importance weight but analytically evaluates the expectation over
$\mu$. Thus we can expect that both are fairly low variance.

For both, we use Bernstein's inequality. For $T_1$, each term is bounded in $[-z_{\infty},z_\infty]$ where $z_{\infty} = \max_{x,a} z(x,a)$ and its variance is
bounded by
$\EE_\mu[z(x,a)^2]$. Thus we get that with probability at least $1-\delta/2$
\begin{align*}
\EE\sbr{z(x,a)(r - \heta(x,a))^2} \leq T_1 + \sqrt{\frac{2 \EE_{\mu}[z(x,a)^2] \log(2/\delta)}{n}} + \frac{2z_{\infty}\log(2/\delta)}{3n}.
\end{align*}

For $T_2$, we similarly to the pessimistic case convert the inner expectation w.r.t.\ $\mu(a\given x_i)$ to an expectation w.r.t.\ $\pi(a\given x_i)$, obtaining a random variable bounded between 0 and $\max_{x,a}w(x,a)/z(x,a)$. Using Hoeffding's inequality, we obtain that with probability $1-2\delta$
\begin{align*}
\EE_\mu\BigBracks{\abs{\hat{w}(x,a) - w(x,a)}^2/z(x,a)} \leq T_2 + \sqrt{\frac{\max_{x,a} \frac{w(x,a)}{z(x,a)} \log(2/\delta)}{2n}}.
\end{align*}
The high probability upper bound follows by multiplying the two right
hand sides together and taking square root.

\section{Experimental Details and Additional Results}
\label{app:experiments}

\subsection{Experimental Details and Results for Atomic Actions}
\paragraph{Dataset statistics.} We use datasets from the \href{https://archive.ics.uci.edu/ml/index.php}{UCI Machine Learning Repository \cite{Dua:2019}}. Dataset statistics are displayed in~\pref{tab:datasets}.

\paragraph{Hyperparameter grid.}
For our shrinkage estimators and \switch, we choose the shrinkage
coefficients from a grid of 30 geometrically spaced values. For the
pessimistic estimator and \switch, the largest and smallest values in
the grid are the $0.05$ quantile and $0.95$ quantile of the importance weights. For the optimistic estimator, the largest and smallest values
are $0.01\times(w_{0.05})^2$ and $100\times(w_{0.95})^2$ where $w_{0.05}$ and
$w_{0.95}$ are the 0.05 and 0.95 quantile of the importance weights.

For the off-policy learning experiments, we only consider the
shrinkage coefficients in $\{0.0,0.1,1,10,100,1000,\infty\}$ during
training, while for model selection, we use the same grid as in the
evaluation experiments.

\paragraph{MRDR.}
\citet{farajtabar2018more} propose training the
regression model with a specific choice of weighting $z$, which we also use in our
experiments. When the evaluation policy $\pi$ is deterministic, they
set $z(x,a) = \one\{\pi(x) = a\}\cdot\frac{1 -\mu(a\given x)}{\mu(a
  \given x)^2}$. For stochastic policies, following the implementation
of Farajtabar et al., we sample $a_i \sim \pi(\cdot \given x_i)$ for
each example in the dataset used to train the reward predictor. Then
we proceed as if the evaluation policy deterministically chooses $a_i$
on example $x_i$.

\paragraph{Ablation study for deterministic target policy.}
Since MRDR is more suited to deterministic policies, we also report
the results of our regressor and shrinkage ablations for a
deterministic target policy $\pi_{1,\textrm{det}}$ in~\pref{tab:ablations_det}. As with the stochastic policies, the
estimator influences the choice of reward predictor, but note that
$z=1$ and $z=w$ are more favorable here. This is likely due to high variance
suffered from training with $z=w^2$, because the importance weights
are larger with a deterministic policy. Our shrinkage
ablation reveals that both estimator types are important also when
the target policy is deterministic.



\paragraph{Ablation study for model selection.}
In~\pref{tab:bias_ab}, we show the comparison of different model
selection methods under different reward predictors and different
shrinkage types. In most cases, Dir-all (\drsdirect where the bias
bound is estimated as the pointwise minimum of (1) the bias, (2) the
optimistic bound and (3) the pessimistic bound) and Up-all (all bias
estimates are adjusted by adding twice standard error before taking pointwise
minimum) are most frequently statistically indistinguishable from the
best, which suggests that our proposed bias estimate (by taking
pointwise minimum of the three) is robust and adaptive.

\paragraph{Comparisons across additional experimental conditions.}
In~\pref{fig:cdfs_scenario} and~\pref{fig:cds_small}, we compare our
new estimators, \drsdirect and \drsup, with baselines across
various conditions (apart from deterministic versus stochastic rewards from the main paper). We first investigate the performance under \emph{friendly logging} (logging and
evaluation policies are derived from the same deterministic policy,
$\pi_{1,det}$), \emph{adversarial logging} (logging and evaluation policies are
derived from different policies $\pi_{1,det}$, $\pi_{2, det}$), and
\emph{uniform logging} (logging policy is uniform over all actions). Then
we plot the performance in the small sample regime, where we aggregate
the 108 conditions (6 logging policies, 9 datasets,
deterministic/stochastic reward) at just 200 bandit samples.

\paragraph{Comparisons across all reward predictors.}

In~\pref{tab:ecdf_sig_mix}--\pref{tab:ecdf_sig_mrdr},
we compare the performance of \drsdirect and \drsup against baselines across various
choices of reward predictors. We begin with using the best reward
predictor type for each method (matching the setting of the main paper),
and then consider each reward predictor in turn, across all estimators.
We report the number of conditions where each estimator is
statistically indistinguishable from the best, and the number of
conditions where each estimator statistically dominates all others.
\drsup is most often in the top group and most often the unique
winner. \drsdirect is also better than \snips, \sndr, and
\switch. These results suggest that our shrinkage estimators are
robust to different choices of reward predictors, and not just limited
to the recommended set $\{\hat{\eta}\equiv0, z=w^2\}$.


\paragraph{Robustness of \drsdirect and \drsup (w.r.t.\ inclusion of more reward predictors)}
In~\pref{fig:robustness}, we test the robustness of our proposed
methods as we incorporate more reward predictors. Our practical
suggestions is to use $\{\hat{\eta} \equiv 0, z = w^2\}$ (shown as
\drsdirect and \drsup in the figure). Here we also evaluate these
methods when selecting from all reward predictors in the set
$\{\hat{\eta}\equiv 0, z\equiv 1, z=w, z=w^2, \text{MRDR}\}$
(shown as \drsdirect(all) and \drsup(all) in the figure). For \drsdirect,
the curves almost match, suggesting that it is quite robust. However,
\drsup is less robust to including additional reward predictors.

\paragraph{Learning curves.} At the end of appendix, we provide learning curves across all conditions. Dataset \emph{glass} is excluded since we only ran it for a single sample size $n=214$.


\begin{table}
\begin{center}
{\small
\begin{tabular}{|c|c|c|c|c|c|c|c|c|c|}
\hline
Dataset & Glass & Ecoli & Vehicle & Yeast & PageBlok & OptDigits & SatImage & PenDigits & Letter\\
\hline\hline
Actions & 6 & 8 & 4 & 10 & 5 & 10 & 6 & 10 & 26\\
\hline
Examples & 214 & 336 & 846 & 1484 & 5473 & 5620 & 6435 & 10992 & 20000\\
\hline
\end{tabular}}
\vspace{-0.25cm}
\caption{Dataset statistics.}
\label{tab:datasets}
\vspace{-0.1cm}
\end{center}
\end{table}

\begin{table}
\begin{minipage}{0.6\textwidth}
\begin{center}
{\small


\begin{tabular}{| l | c | c | c | c | c |}
\hline
& $\hat{\eta} \equiv 0$ & $z \equiv 1$ & $z = w$ & $z = w^2$ & MRDR\\
\hline\hline
DM & 0 (0) & 54 (30) & 59 (23) & 35 (4) & 24 (6)\\
DR & 28 (1) & 94 (11) & 85 (0) & 85 (1) & 85 (0)\\
snDR & 65 (7) & 86 (7) & 79 (0) & 72 (0) & 71 (0)\\
DRs & 14 (9) & 51 (17) & 65 (14) & 54 (6) & 47 (4)\\
\hline
\end{tabular}
}\end{center}
\end{minipage}
\begin{minipage}{0.4\textwidth}
\begin{center}
{\small
\begin{tabular}{| l | c | c |}
\hline
& \lone & \ltwo \\
\hline\hline
$\hat{\eta} \equiv 0$ & 13 & 59\\
$z \equiv 1$ & 29 & 55\\
$z = w$ & 26 & 66\\
$z = w^2$ & 30 & 67\\
\textsc{MRDR} & 29 & 63\\
\hline
\end{tabular}
}\end{center}
\end{minipage}
\vspace{-0.35cm}
\caption{Ablation analysis for
  \emph{deterministic} target policy $\pi_{1,\textrm{det}}$ across
  experimental conditions. Left: we compare reward predictors using a
  fixed estimator (with oracle tuning if applicable). We report the
  number of conditions where a regressor is statistically
  indistinguishable from the best and, in parenthesis, the number of
  conditions where it statistically dominates all others. Right: we
  compare different shrinkage types using a fixed reward predictor
  (with oracle tuning) reporting the number of conditions where one
  statistically dominates the other.\looseness=-1}
\label{tab:ablations_det}
\end{table}

\subsection{Experimental Details for Combinatorial Actions}
\label{app:exp_slates}

\paragraph{Hyperparameter grid.} We select the hyperparameter $\lambda$ from the grid of 15 geometrically spaced values, with the smallest value $0.01\times (w_{0.05})^2$ and the largest value $100\times (w_{0.95})^2$, where $w_{0.05}$ and $w_{0.95}$ are the 0.05 and 0.95 quantiles of the weights $w(x_i,\ab_i)$ on the logged data. We also add two boundary values $\lambda = 10^{-50}$ and $\lambda=10^{30}$ to include DM and DR-PI as special cases.

\paragraph{Basis construction.}
We use the logging distribution supported on a linearly independent set of actions, i.e., a basis, constructed following Algorithm~\ref{alg:example}. The number of elements of the basis for the action space of lists of length $\ell$ out of $m$ items is $1+\ell(m-1)$.

\begin{algorithm}
   \caption{Constructing basis for the action space of lists of length $\ell$ out of $m$ items.}
   \label{alg:example}
\small
\begin{algorithmic}
   \STATE \textit{Actions $\ab$ are represented as tuples of size $\ell$
                  with entries $a[i]\in\set{0,\dots,m-1}$, indexed by $i\in\set{0,\dotsc,\ell-1$}.}
   \smallskip
   \STATE \textit{Actions $\ab$ correspond to vectors in $\RR^{\ell m}$, obtained by
                  representing each $a[i]$
                  as a vector of standard basis in $\RR^m$,}
   \STATE \textit{and concatenating these vectors.}
   \smallskip
   \STATE {\bfseries Assume:} Greedy action is the tuple $\gb=[1,2,\dots,\ell]$
   \smallskip
   \STATE Initialize $\Bcal=\set{\gb}$
   \FOR{$i=1$ {\bfseries to} $\ell-1$}
   \STATE Set $\ab=\gb$
   \STATE Set $a[0]=i$ and $a[i]=0$, $\Bcal=\Bcal\cup \{\ab\}$
   \ENDFOR
   \FOR{$j=\ell$ {\bfseries to} $m-1$}
   \STATE Set $\ab=\gb$
   \STATE Set $a[0]=j$ and $\Bcal=\Bcal\cup \{\ab\}$
   \ENDFOR
   \FOR{$i=1$ {\bfseries to} $\ell-1$}
   \FOR{$i'=1$ {\bfseries to} $\ell-1$ such that $i\neq i'$}
   \STATE Set $\ab=\gb$
   \STATE Set $a[0]=i$, $a[i]=i'$ and $a[i']=0$, $\Bcal=\Bcal\cup \{\ab\}$
   \ENDFOR
   \FOR{$j=\ell$ {\bfseries to} $m-1$}
   \STATE Set $\ab=\gb$
   \STATE Set $a[0]=i$, $a[i]=j$ and $\Bcal=\Bcal\cup \{\ab\}$
   \ENDFOR
   \ENDFOR
   \FOR{$i=1$ {\bfseries to} $\ell-1$}
   \STATE Set $\ab=\gb$
   \STATE Set $a[0]=\ell$, $a[i]=0$ and $\Bcal=\Bcal\cup \{\ab\}$
   \ENDFOR
   \STATE Return $\Bcal$
\end{algorithmic}
\end{algorithm}


\begin{table}
\begin{center}
\begin{tabular}{| l | c | c | c | c | c | c | c | c |}
\hline
& Dir-all & Dir-naive & Dir-opt & Dir-pes & Up-all & Up-naive & Up-opt & Up-pes \\
\hline\hline
0-pes & 71 & 67 & 74 & 78 & 63 & 60 & 79 & 79\\
0-opt & 75 & 68 & 71 & 81 & 64 & 62 & 66 & 81\\
0-best & 63 & 59 & 67 & 74 & 56 & 54 & 64 & 76\\
$w^2$-pes & 47 & 41 & 5 & 3 & 72 & 71 & 5 & 3\\
$w^2$-opt & 47 & 40 & 3 & 2 & 73 & 70 & 3 & 2\\
$w^2$-best & 47 & 42 & 4 & 3 & 75 & 72 & 4 & 3\\
best-pes & 51 & 46 & 7 & 5 & 69 & 70 & 6 & 5\\
best-opt & 49 & 46 & 7 & 3 & 69 & 70 & 5 & 3\\
best-best & 50 & 46 & 8 & 4 & 71 & 70 & 6 & 4\\
all-best & 49 & 46 & 7 & 6 & 65 & 67 & 6 & 6\\
\hline
\end{tabular}
\caption{Comparison of model selection methods when paired with
  different reward predictor sets and shrinkage types. As in other
  tables, we record the number of conditions in which this model
  selection method is statistically indistinguishable from the best,
  for fixed reward predictor set and shrinkage types. Columns are
  indexed by model selection methods, ``Dir'' denotes taking sample
  average and ``Up'' denotes inflating sample averages with twice the
  standard error. ``Naive'' denotes directly estimating bias, ``opt''
  denotes estimating optimistic bias bound, ``pes'' denotes
  pessimistic bias bound, and ``all'' denotes taking the pointwise
  minimum of all three. Rows are indexed by reward predictors:
  $\hat{\eta}\equiv 0$, $z = w^2$, ``best'' denotes selecting over
  both, and ``all'' denotes selecting over these and additionally
  $z=1$, $z=w$, and MRDR. Rows are also indexed by shrinkage type,
  optimistic, pessimistic, and best, which denotes model selection
  over both.}
\label{tab:bias_ab}
\end{center}
\end{table}

\begin{figure}
\begin{center}
\includegraphics[width=0.5\textwidth]{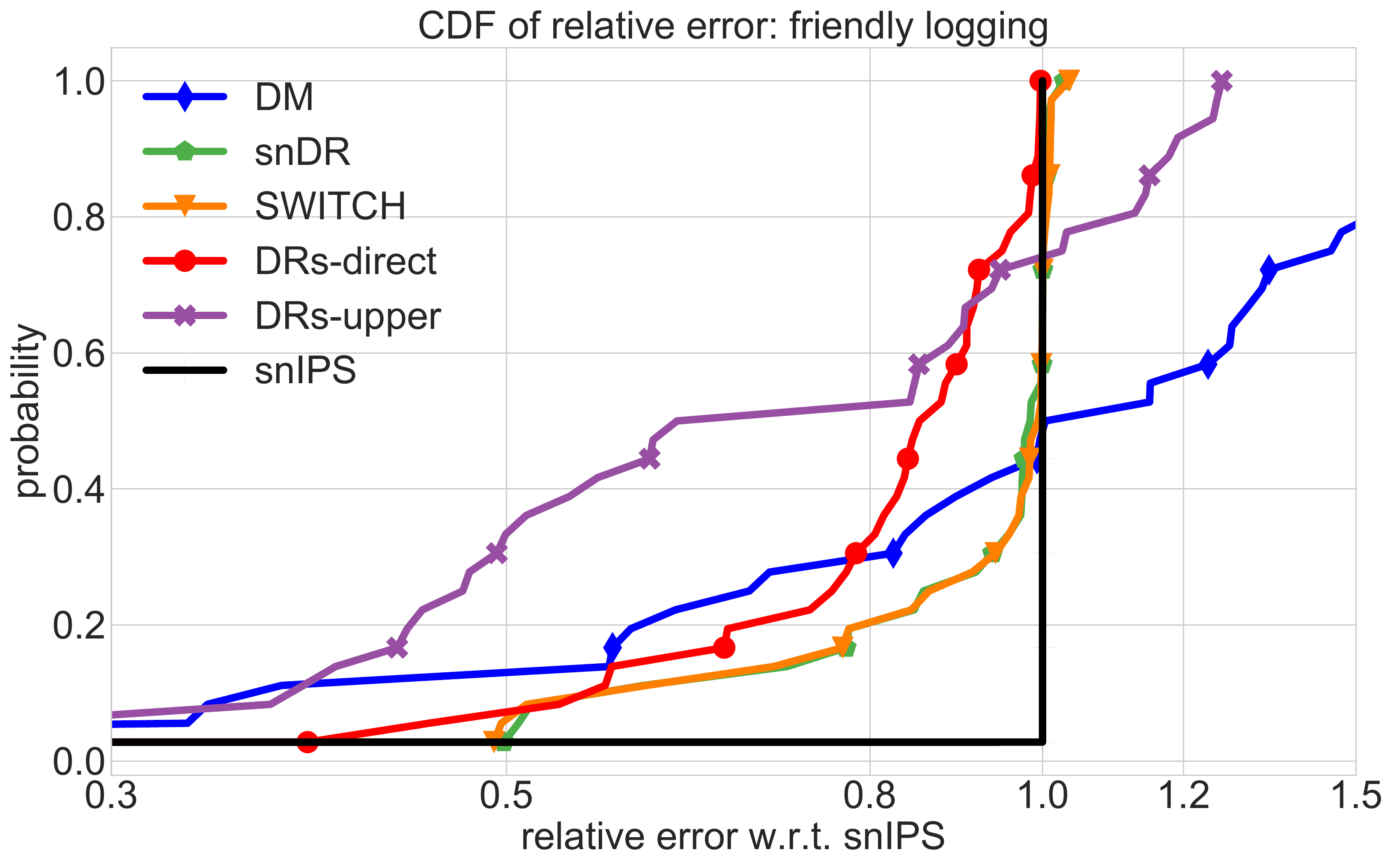}
\hspace{-0.2cm}
\includegraphics[width=0.5\textwidth]{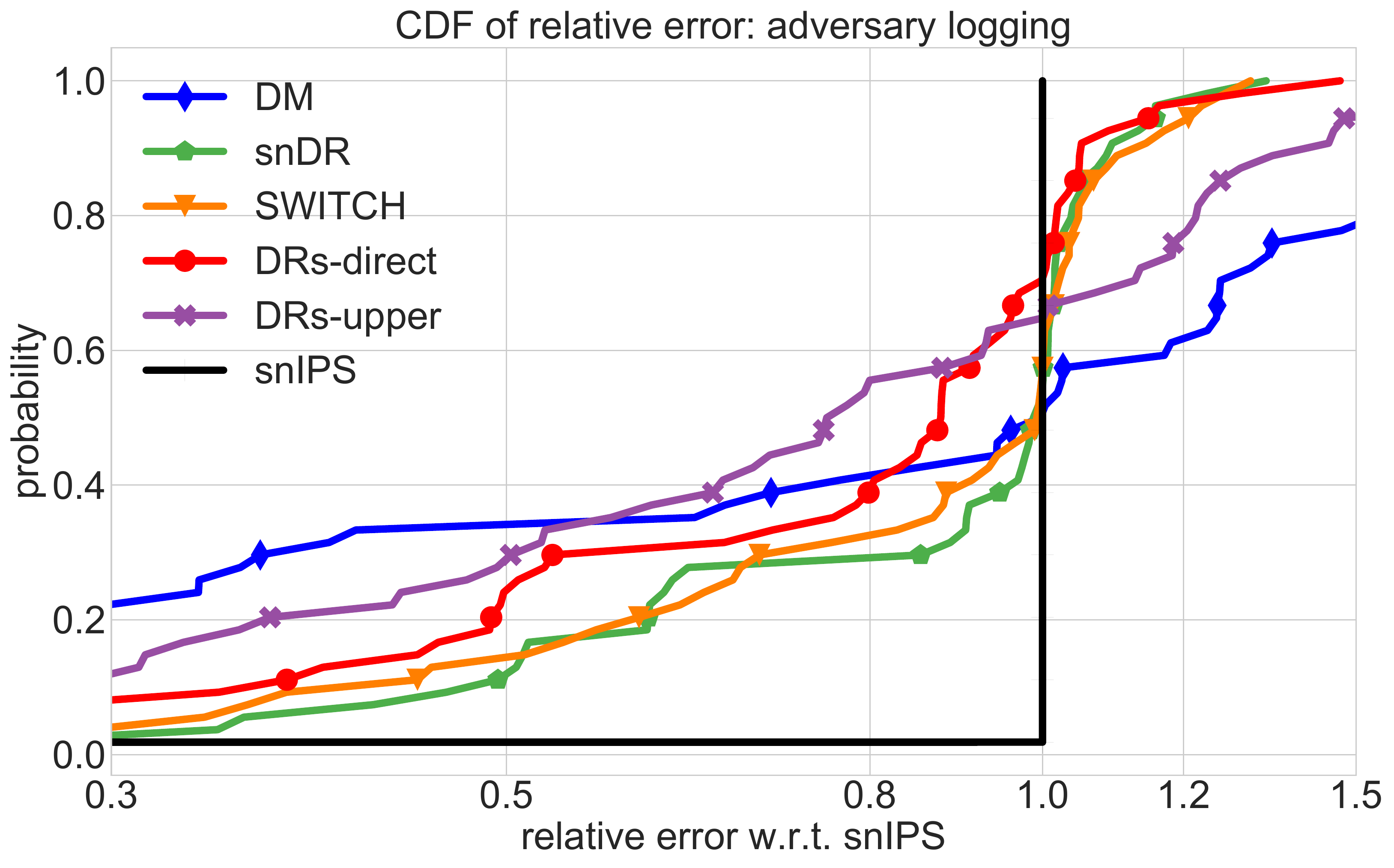}
\end{center}
\vspace{-0.5cm}
\caption{CDF plots of normalized MSE aggregated across all conditions
  with friendly scenario (left) and adversary scenario
  (right). }
\label{fig:cdfs_scenario}
\end{figure}
\begin{figure}
\begin{center}
\includegraphics[width=0.5\textwidth]{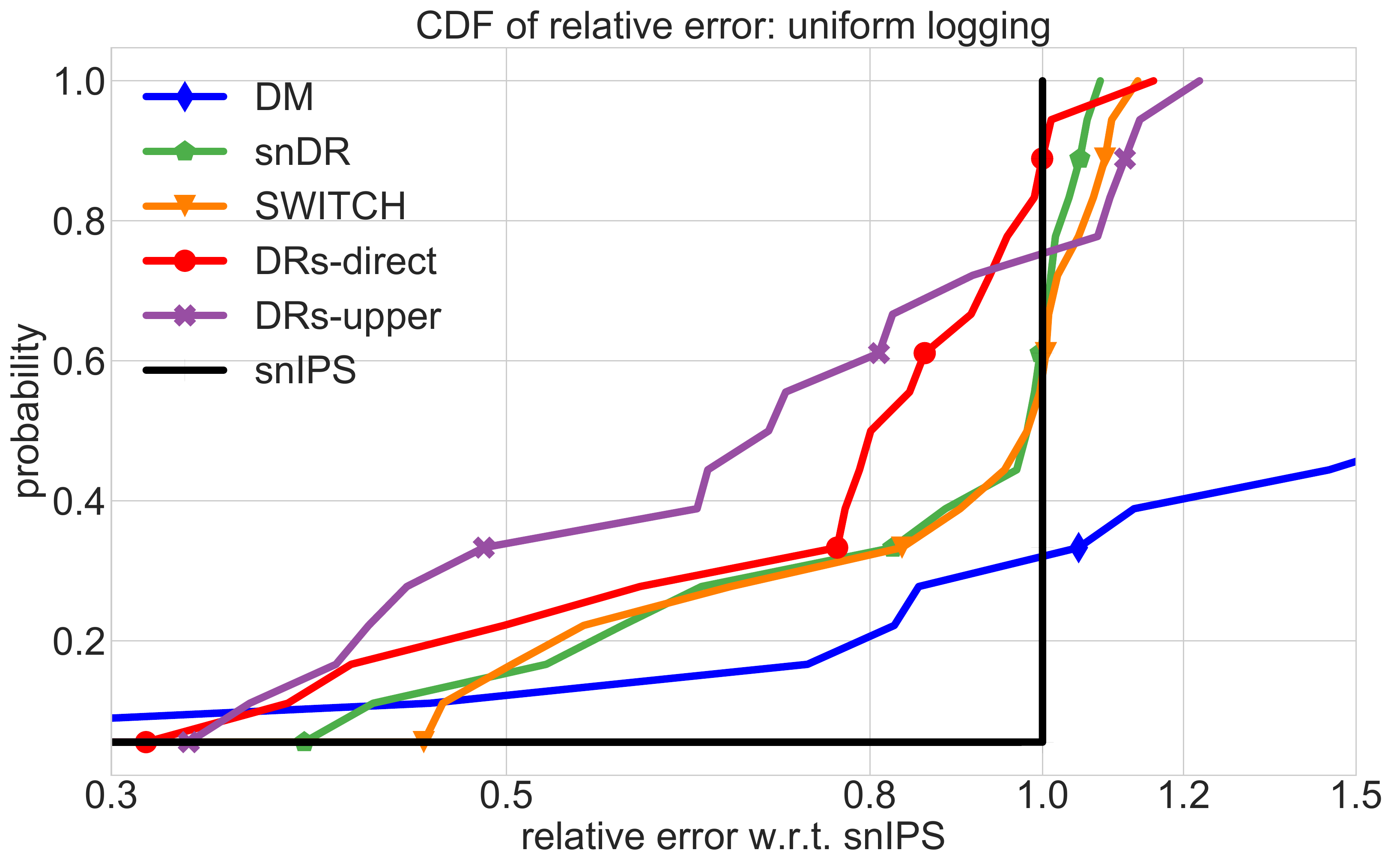}
\hspace{-0.2cm}
\includegraphics[width=0.5\textwidth]{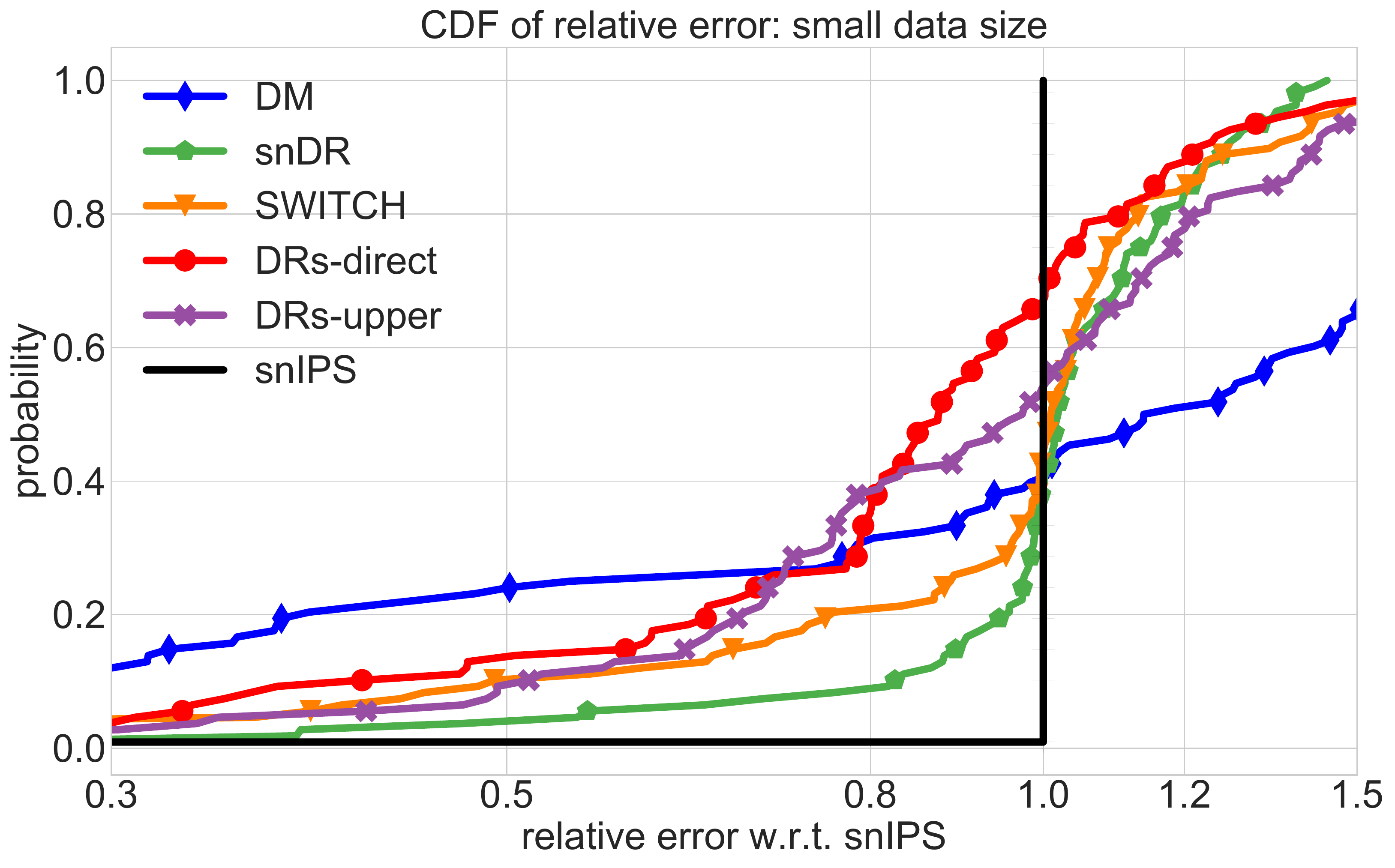}
\end{center}
\vspace{-0.5cm}
\caption{CDF plots of normalized MSE aggregated across all conditions
  with uniform logging policy scenario (left) and small data regime
  (right). }
\label{fig:cds_small}
\end{figure}

\begin{figure}
\begin{center}
\includegraphics[width=0.6\textwidth]{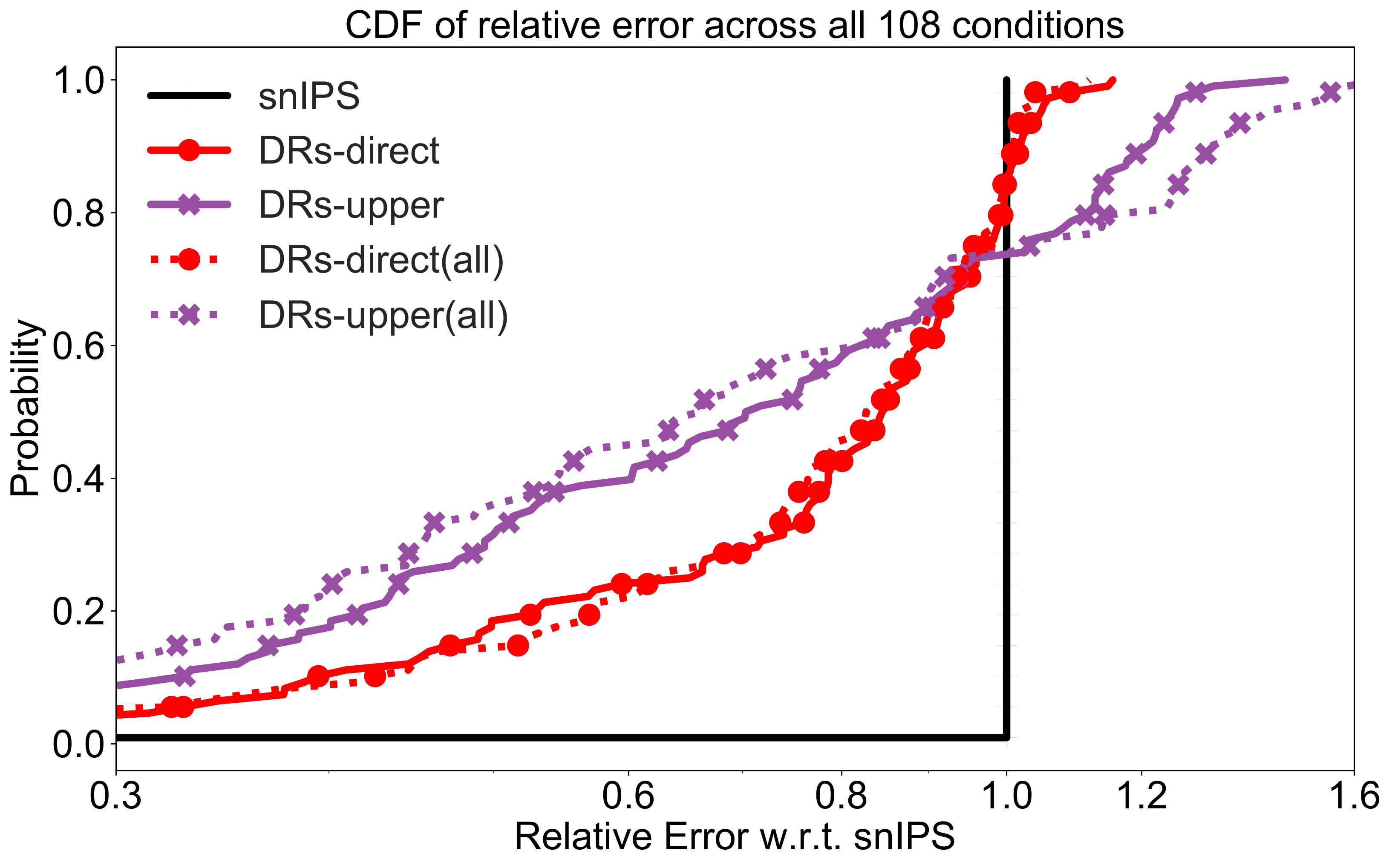}
\end{center}
\vspace{-0.5cm}
\caption{Robustness test for \drsdirect and \drsup. \drsdirect(all) and \drsup(all) means the corresponding method with reward predictor select from all possible cases $\{\hat{\eta} = 0, z=1, z=w, z=w^2\}$ and MRDR.}
\label{fig:robustness}
\end{figure}

\begin{table}
\begin{center}{\small\vspace{0.1cm}
\begin{tabular}{|l | c | c | c | c | c | c |}
\hline
& \snips & DM & \sndr & \switch & \drsdirect & \drsup\\
\hline\hline
Best or Tied & 1 & 39 & 1 & 1 & 27 & 56\\
Unique Best & 0 & 28 & 0 & 0 & 14 & 52\\
\hline
\end{tabular}}
\caption{Significance testing for different estimators across all
  conditions, with each estimator using its best reward predictor (DM uses $z=1$, snDR uses $z=w$, while SWITCH and DRs use reward from $\{\hat{\eta}=0, z=w^2\}$). In the top row, we report the number of conditions where
  each estimator is statistically indistinguishable from the best, and
  in the bottom row we report the number of conditions where each
  estimator is the uniquely best.}
\label{tab:ecdf_sig_mix}
\end{center}
\end{table}

\begin{table}
\begin{center}{\small\vspace{-0.1cm}
\begin{tabular}{|l | c | c | c | c | c | c |}
\hline
& \snips & DM & \sndr & \switch & \drsdirect & \drsup\\
\hline\hline
Best or Tied & 10 & 39 & 19 & 15 & 25 & 49\\
Unique Best & 0 & 23 & 3 & 1 & 5 & 43\\
\hline
\end{tabular}}
\caption{Significance testing for different estimators across all
  conditions (DM, snDR use reward $z=1$, while SWITCH and DRs use reward from $\{\hat{\eta}=0, z=1\}$). In the top row, we report the number of conditions where
  each estimator is statistically indistinguishable from the best, and
  in the bottom row we report the number of conditions where each
  estimator is the uniquely best.}
\label{tab:ecdf_sig_w0}
\end{center}
\end{table}

\begin{table}
\begin{center}{\small\vspace{-0.1cm}
\begin{tabular}{|l | c | c | c | c | c | c |}
\hline
& \snips & DM & \sndr & \switch & \drsdirect & \drsup\\
\hline\hline
Best or Tied & 3 & 44 & 5 & 5 & 30 & 58\\
Unique Best & 0 & 23 & 0 & 0 & 6 & 51\\
\hline
\end{tabular}}
\caption{Significance testing for different estimators across all
  conditions (DM, snDR use reward $z=w$, while SWITCH and DRs use reward from $\{\hat{\eta}=0, z=w\}$). In the top row, we report the number of conditions where
  each estimator is statistically indistinguishable from the best, and
  in the bottom row we report the number of conditions where each
  estimator is the uniquely best.}
\label{tab:ecdf_sig_w1}
\end{center}
\end{table}

\begin{table}
\begin{center}{\small\vspace{-0.1cm}
\begin{tabular}{|l | c | c | c | c | c | c |}
\hline
& \snips & DM & \sndr & \switch & \drsdirect & \drsup\\
\hline\hline
Best or Tied & 5 & 36 & 5 & 4 & 43 & 63\\
Unique Best & 0 & 15 & 0 & 0 & 13 & 46\\
\hline
\end{tabular}}
\caption{Significance testing for different estimators across all
  conditions (DM, snDR use reward $z=w^2$, while SWITCH and DRs use reward from $\{\hat{\eta}=0, z=w^2\}$). In the top row, we report the number of conditions where
  each estimator is statistically indistinguishable from the best, and
  in the bottom row we report the number of conditions where each
  estimator is the uniquely best.}
\label{tab:ecdf_sig}
\end{center}
\end{table}

\begin{table}
\begin{center}{\small\vspace{-0.1cm}
\begin{tabular}{|l | c | c | c | c | c | c |}
\hline
& \snips & DM & \sndr & \switch & \drsdirect & \drsup\\
\hline\hline
Best or Tied & 10 & 17 & 8 & 10 & 37 & 73\\
Unique Best & 0 & 7 & 0 & 0 & 17 & 57\\
\hline
\end{tabular}}
\caption{Significance testing for different estimators across all
  conditions (DM, snDR use reward estimated from MRDR, while SWITCH and DRs use reward from $\{\hat{\eta}=0, \text{MRDR}\}$). In the top row, we report the number of conditions where
  each estimator is statistically indistinguishable from the best, and
  in the bottom row we report the number of conditions where each
  estimator is the uniquely best.}
\label{tab:ecdf_sig_mrdr}
\end{center}
\end{table}

\clearpage
\begin{figure}
\begin{center}
\includegraphics[width=1\textwidth, height=11cm]{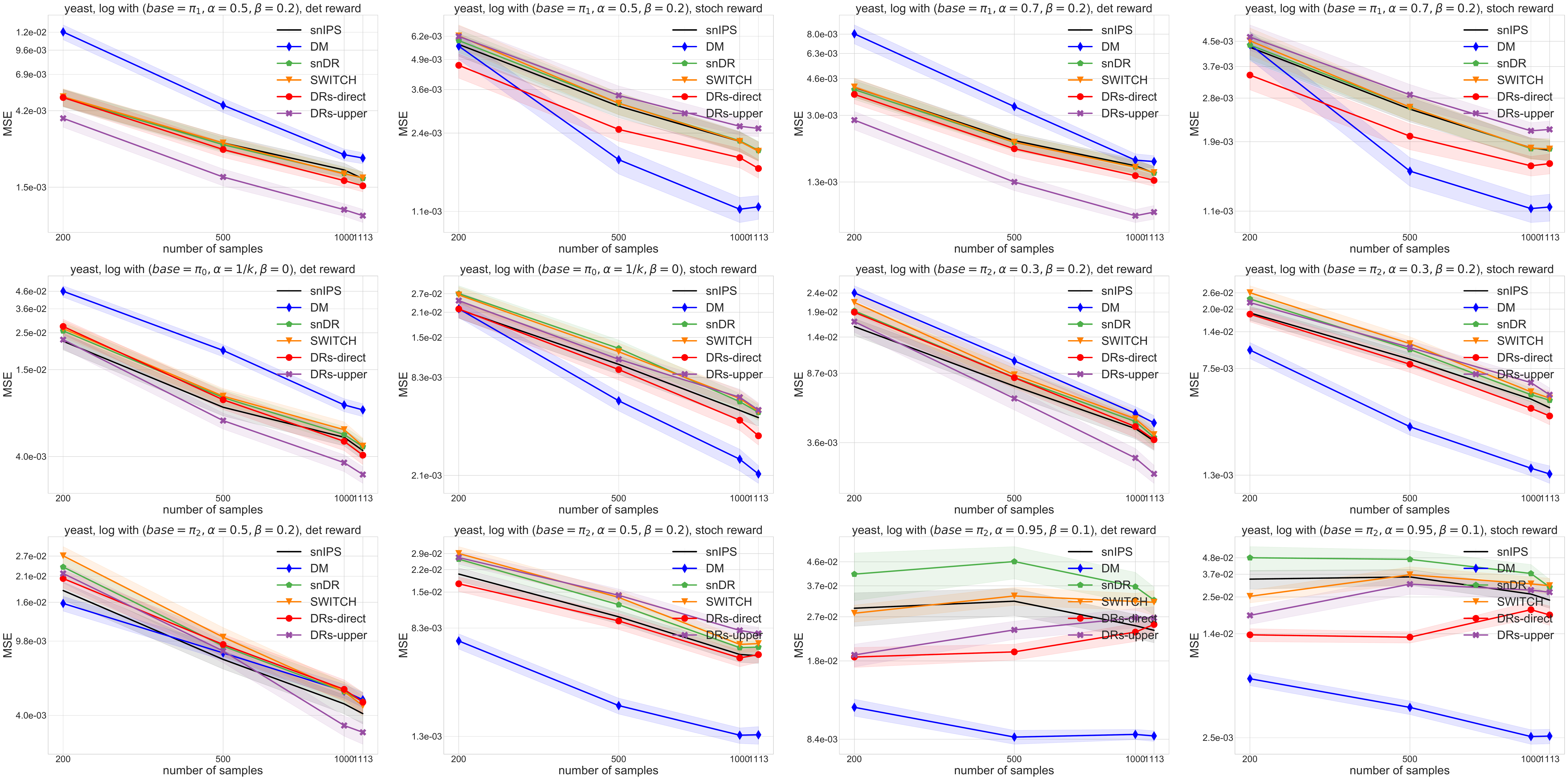}
\hspace{-0.2cm}
\end{center}
\vspace{-0.5cm}
\end{figure}

\begin{figure}
\begin{center}
\includegraphics[width=1\textwidth, height=11cm]{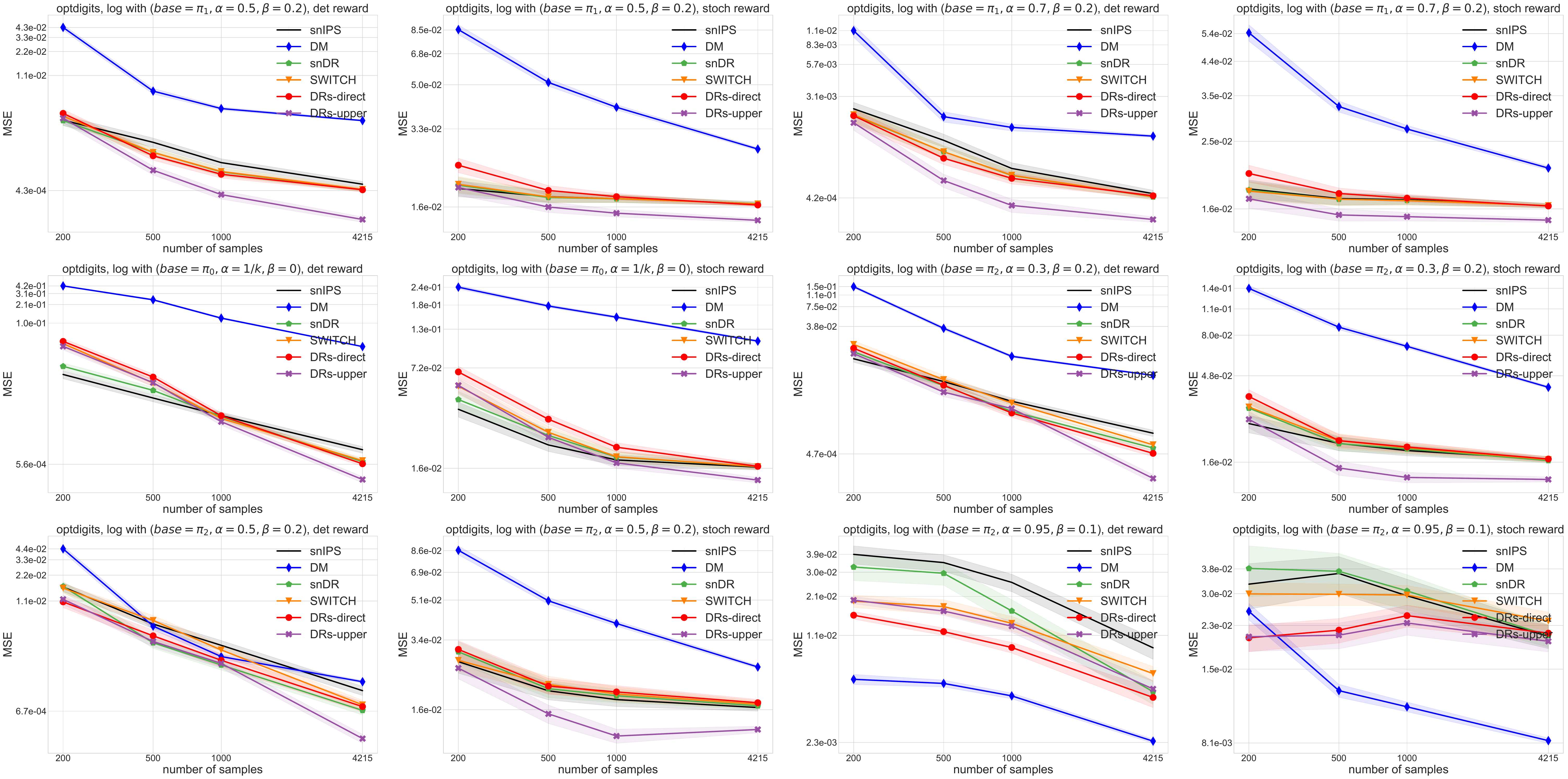}
\hspace{-0.2cm}
\end{center}
\vspace{-0.5cm}
\end{figure}

\begin{figure}
\begin{center}
\includegraphics[width=1\textwidth, height=11cm]{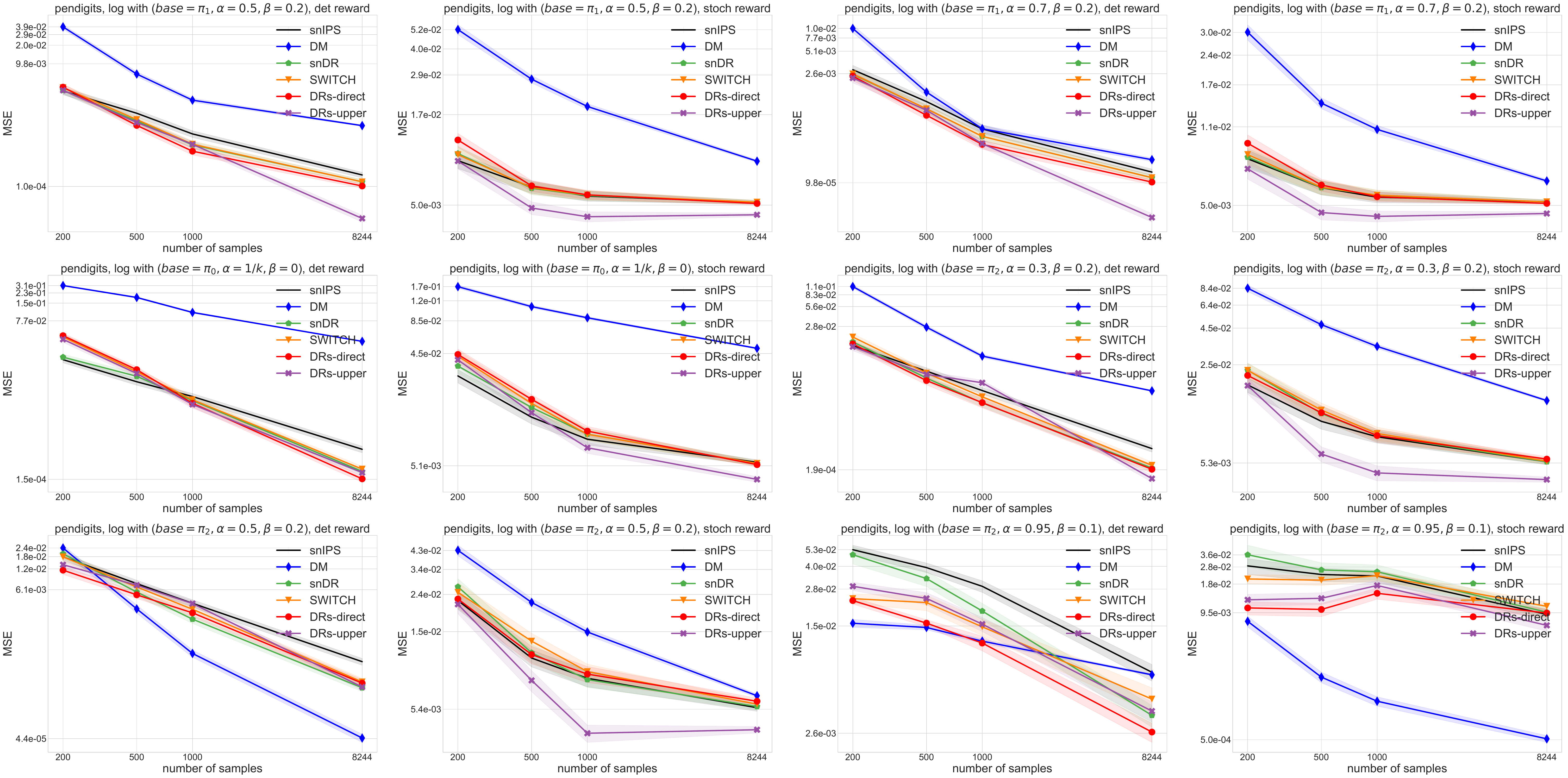}
\hspace{-0.2cm}
\end{center}
\vspace{-0.5cm}
\end{figure}

\begin{figure}
\begin{center}
\includegraphics[width=1\textwidth, height=11cm]{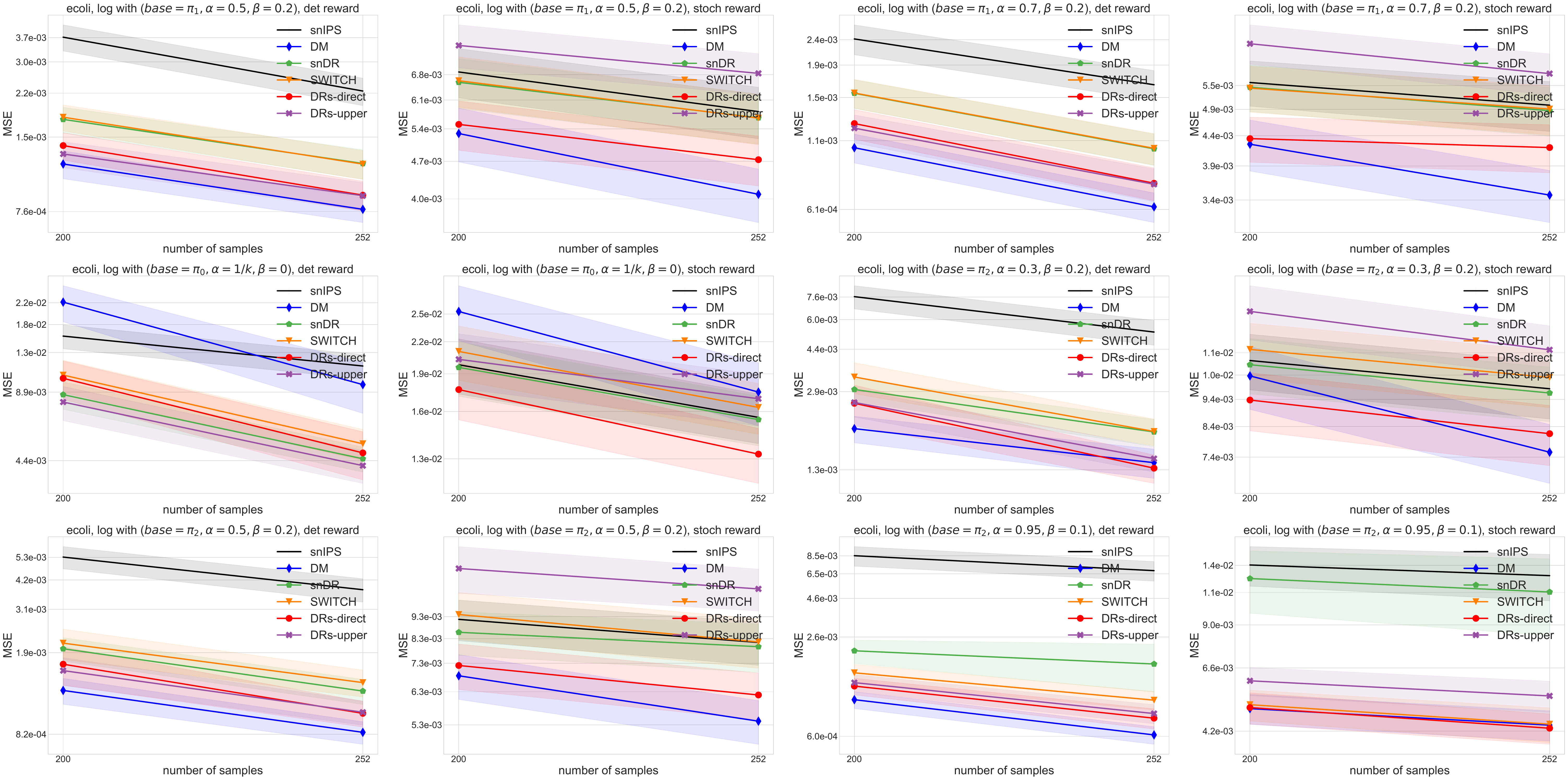}
\hspace{-0.2cm}
\end{center}
\vspace{-0.5cm}
\end{figure}

\begin{figure}
\begin{center}
\includegraphics[width=1\textwidth, height=11cm]{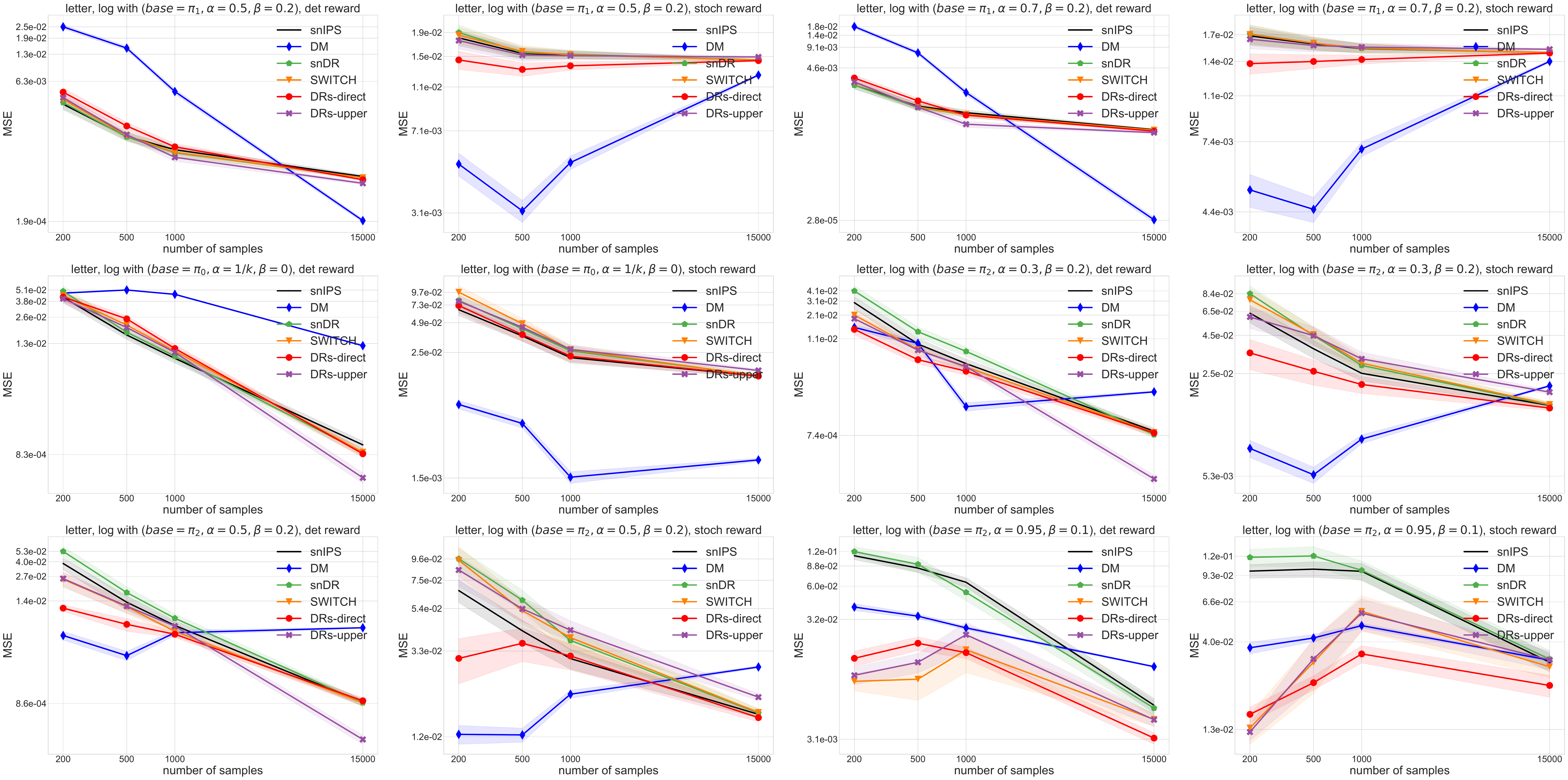}
\hspace{-0.2cm}
\end{center}
\vspace{-0.5cm}
\end{figure}

\begin{figure}
\begin{center}
\includegraphics[width=1\textwidth, height=11cm]{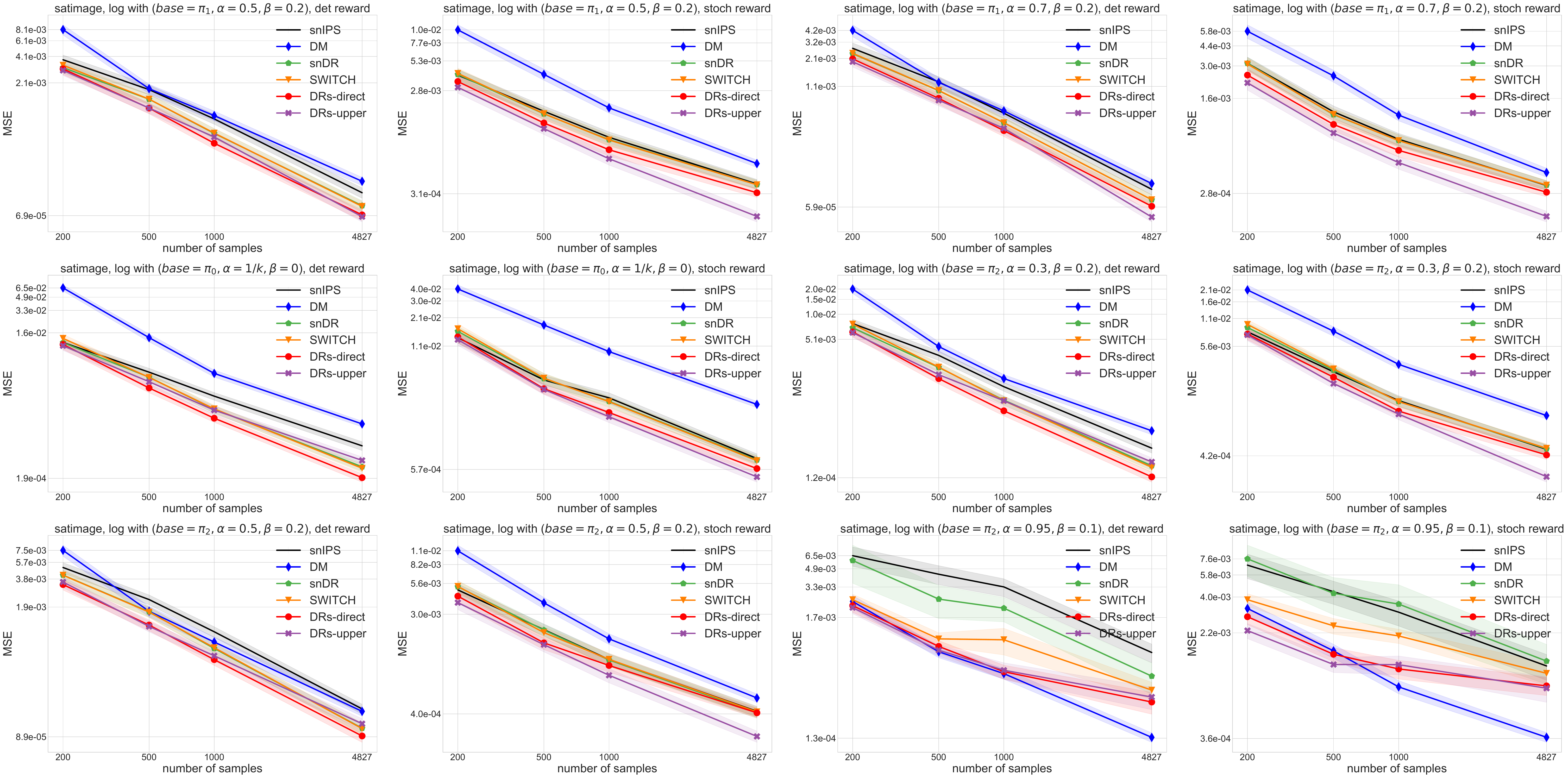}
\hspace{-0.2cm}
\end{center}
\vspace{-0.5cm}
\end{figure}

\begin{figure}
\begin{center}
\includegraphics[width=1\textwidth, height=11cm]{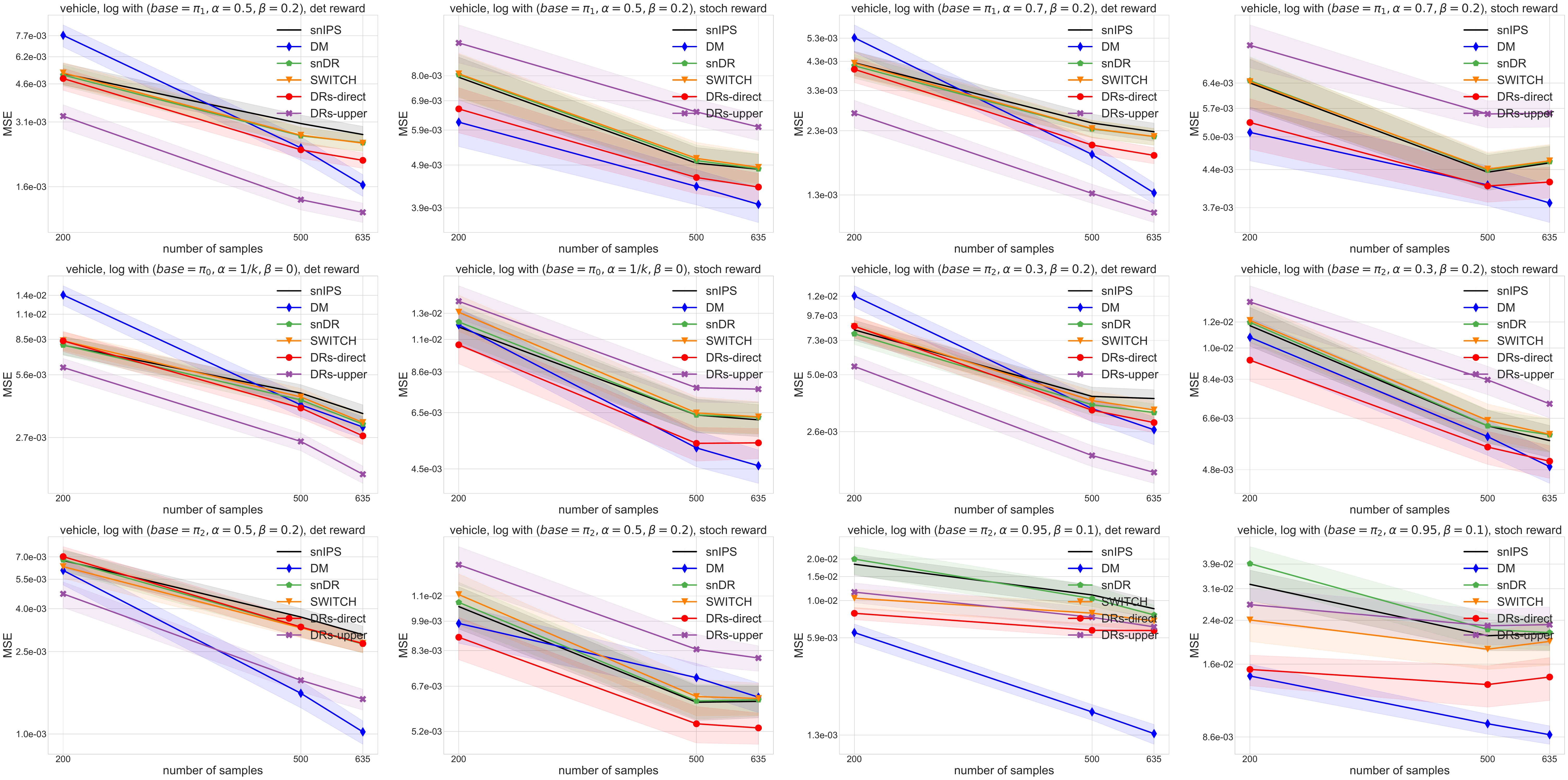}
\hspace{-0.2cm}
\end{center}
\vspace{-0.5cm}
\end{figure}

\begin{figure}
\begin{center}
\includegraphics[width=1\textwidth, height=11cm]{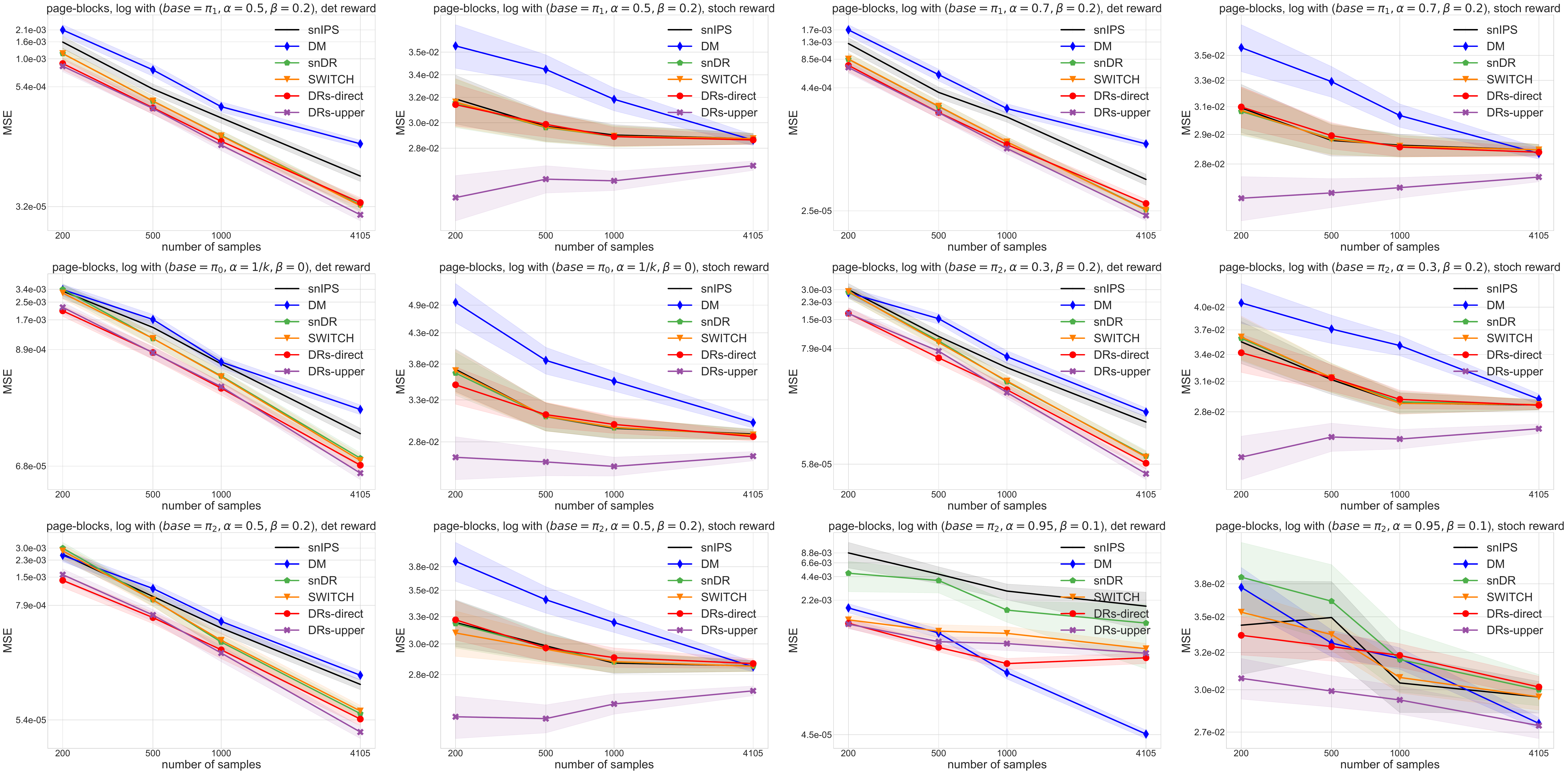}
\hspace{-0.2cm}
\end{center}
\vspace{-0.5cm}
\end{figure}

\end{document}